%% file: main.tex
\documentclass{article}

 \usepackage[preprint]{neurips_2026}

\usepackage[utf8]{inputenc} 
\usepackage[T1]{fontenc}    
\usepackage[hidelinks]{hyperref}       
\usepackage{url}            
\usepackage{graphicx}
\usepackage{booktabs}       
\usepackage{amsfonts}       
\usepackage{nicefrac}       
\usepackage{microtype}      
\usepackage{xcolor}         
\usepackage{amsmath}
\usepackage{mathtools}
\usepackage{multirow}
\usepackage{multicol}
\usepackage{algorithm}
\usepackage{algorithmic}

\usepackage{amsmath}
\usepackage{amssymb}
\usepackage{mathtools}
\usepackage{amsthm}
\usepackage{enumitem}
\usepackage{xcolor}

\theoremstyle{plain}
\newtheorem{theorem}{Theorem}[section]
\newtheorem{proposition}[theorem]{Proposition}

\theoremstyle{definition}

\theoremstyle{remark}

\newcommand{\dsafe}{D_{\mathrm{safe}}}
\newcommand{\dsft}{D_{\mathrm{SFT}}}

\usepackage{xcolor}
\usepackage{empheq}

\definecolor{mygray}{gray}{0.9}

\title{Beyond Average Safety: Chance-Constrained LLM Fine-tuning}

\author{%
  Taha Entesari, Mahyar Fazlyab\\
  Department of Electrical and Computer Engineering\\
  Johns Hopkins University\\
  \texttt{\{tentesa1, mahyarfazlyab\}@jhu.edu} \\
}

\begin{document}

\maketitle

\begin{abstract}
Fine-tuning large language models on new objectives can improve helpfulness, instruction following, or domain-specific performance, but it can also induce regressions on safety-critical prompts. Existing safety-preserving fine-tuning methods typically control average safety loss or use weighted auxiliary penalties, which can obscure rare but severe failures. We propose a chance-constrained formulation for safety-preserving fine-tuning that limits the fraction of safety examples whose degradation relative to a reference model exceeds a prescribed threshold. Because the resulting empirical chance constraint contains a discontinuous indicator, we introduce a differentiable majorization of the violation rate, yielding a tractable conservative constraint. We then develop a constraint-aware gradient descent method that treats the majorized constraint as a safe set in parameter space and minimally modifies the fine-tuning direction to preserve feasibility. The resulting update admits a closed form and produces a tail-aware safety correction that emphasizes examples near or above the degradation threshold. We conduct an extensive set of experiments on harmful fine-tuning across three different tasks and three models and show that our approach consistently outperforms the baselines that exist in the literature.
These results suggest that safety preservation in LLM fine-tuning is better viewed as a reliability-constrained optimization problem than as average-risk regularization.
\end{abstract}

\section{Introduction}

Large language models (LLMs) are increasingly adapted to downstream
objectives through supervised fine-tuning, preference optimization,
and other post-training procedures
\citep{ouyang2022training,rafailov2023direct}. In many practical
settings, however, improving a model on a target fine-tuning dataset
can inadvertently degrade its behavior on safety-sensitive prompts:
a model fine-tuned to become more helpful, domain-specialized, or
instruction-following may regress on prompts involving harmful
requests, refusal behavior, or other safety-critical interactions
\citep{qi2023fine,yang2023shadow,lermen2023lora,zhan2024removing}.
This phenomenon,  often referred to as harmful fine-tuning, 
persists even when the fine-tuning data is benign, and has motivated
a growing body of defenses operating at the alignment stage
\citep{huang2024vaccine,huang2024booster}, at the fine-tuning stage
\citep{huang2024lisa,li2025salora,yi2025gradient,yang2026asft}, in
both stages jointly \citep{nguyen2026antibody}, or as a post-hoc
patch \citep{hsu2024safe}. The shared goal of these defenses is to
resolve a fundamental tension in post-training: how to improve
performance on a desired fine-tuning task while preserving acceptable
behavior on a safety dataset.

A common approach is to combine the fine-tuning loss and the safety
loss through weighted averaging or regularization
\citep{huang2024lisa,yang2026asft,nguyen2026antibody}. While simple
and practical, this strategy controls only average behavior over
the safety dataset. This can be misleading: a model may satisfy an
\textit{average} safety constraint while still exhibiting unacceptable
degradation on a small but important subset of safety examples.
Thus, controlling the mean safety loss does not directly control
the \textit{frequency} of safety regressions.

This motivates viewing safety preservation as a \emph{reliability
constraint} rather than an auxiliary objective. Instead of bounding
the average safety degradation, we require that the degradation
remain below an acceptable threshold on all but a small
user-specified fraction of safety examples. We formalize this
through a chance-constrained formulation
\citep{charnes1959chance,nemirovski2006convex}, capturing a
reliability requirement that expectation-based penalties cannot
express.

The empirical chance constraint contains a discontinuous indicator function, which is
unsuitable for gradient-based fine-tuning. We replace it by a
differentiable majorizer, yielding a conservative surrogate on the
safety-regression rate; different choices trade smoothness for
sharpness, and an exponential majorizer in particular connects the
formulation to entropic risk \citep{follmer2002convex}, paralleling
the role of CVaR in risk-averse optimization
\citep{rockafellar2000optimization}.

To solve the resulting problem, we develop a \emph{constraint-aware
gradient descent} method that treats the majorized chance
constraint as a safe set in parameter space and, at each iteration,
computes the closest Euclidean direction to the nominal fine-tuning
gradient that satisfies a forward-invariance condition for this
safe set --- borrowed from the control barrier function literature
\citep{ames2019control}. The update coincides with standard SFT
when already feasible and applies the smallest correction needed to
remain safe otherwise, yielding a closed-form gradient filter
without a tuned penalty weight or a reactive dual variable.

This perspective differs from standard regularized or multi-task
fine-tuning and is closer in spirit to constrained-policy approaches
in safe reinforcement learning \citep{achiam2017constrained} and to
recent constrained formulations of LLM unlearning
\citep{entesari2025constrained}. Combining chance constraints,
differentiable majorization, and constraint-aware gradient descent
yields a principled mechanism for improving fine-tuning performance
while limiting safety regressions.

In summary, our main contributions are as follows:
\begin{itemize}[leftmargin=*]
    \item We formulate safety-preserving LLM fine-tuning as a
    chance-constrained optimization problem that explicitly bounds
    the fraction of safety examples whose degradation relative to a
    reference model exceeds a prescribed threshold.
    \item We introduce a differentiable majorization of the
    empirical chance constraint, yielding a tractable conservative
    constraint on the safety-regression rate and connecting
    exponential majorization to entropic risk.
    \item We develop a  constraint-aware gradient descent
    method that minimally modifies the nominal fine-tuning update to
    satisfy a forward-invariance condition for the majorized safety
    constraint.
    \item We empirically demonstrate that the proposed method
    improves fine-tuning performance while reducing safety-regression
    rates compared with standard weighted and constrained baselines.
    We showcase the superiority of our methodology through experiments using Qwen 3.5 4B and 9B and Llama 3.1 8B on three highly poisoned tasks (SST-2, AG News, GSM8K).
\end{itemize}

Overall, the paper advocates a simple thesis: \emph{safety
preservation in LLM fine-tuning should be treated as
reliability-constrained optimization rather than average-risk
regularization}. This shift from mean safety loss to controlled
safety-regression rate leads to a practical fine-tuning algorithm
that is both tail-aware and compatible with standard gradient-based
training.

\section{Fine-tuning with Safety Preservation}
\label{sec:setup}

We consider the post-training of a pretrained autoregressive language model $\pi_{\theta_0}(y \mid x)$ into a new model $\pi_\theta(y \mid x)$ using two datasets with distinct roles. The first dataset, denoted by $\dsft=\{(x_j,y_j)\}_{j=1}^m$, is a \emph{fine-tuning dataset} that encodes the target behavior we wish to improve, such as helpfulness, instruction following, or domain-specific competence. The second dataset, $\dsafe=\{(\tilde x_i,\tilde y_i)\}_{i=1}^n$, is a \emph{safety dataset} that captures behaviors we wish to preserve during post-training, such as refusals, safe completions, or other safety-sensitive responses. Our goal is to improve performance on $\dsft$ while avoiding unacceptable degradation on $\dsafe$.

Throughout the paper, we focus on the supervised setting in which both fine-tuning and safety are measured through cross-entropy losses. For an input-output pair $(x,y)$, the autoregressive model defines
$
\pi_\theta(y\mid x)
=
\prod_{t=1}^{T(y)}
\pi_\theta(y_t \mid x,y_{<t}),
$
where $T(y)$ denotes the sequence length. The sequence-level negative log-likelihood is therefore
$
-\log \pi_\theta(y\mid x)
=
-\sum_{t=1}^{T(y)}
\log \pi_\theta(y_t \mid x,y_{<t}).
$
We define the empirical fine-tuning loss as
\begin{equation}
\mathcal{L}_{\mathrm{SFT}}(\theta)
:=
\frac{1}{m}\sum_{j=1}^m \ell^{\mathrm{SFT}}_j(\theta),
\qquad
\ell^{\mathrm{SFT}}_j(\theta)
:=
-\log \pi_\theta(y_j\mid x_j).
\label{eq:sft_loss}
\end{equation}

\paragraph{Safety degradation.}
For each safety example $(\tilde x_i,\tilde y_i)\in \dsafe$, define the safety cross-entropy loss
\[
\ell_i^{\mathrm{safe}}(\theta)
:=
-\log \pi_\theta(\tilde y_i\mid \tilde x_i).
\]
In post-training, the relevant quantity is often not the absolute safety loss, but its degradation relative to the reference model $\pi_{\theta_0}$. We therefore define
\begin{equation}
\Delta \ell_i^{\mathrm{safe}}(\theta)
:=
\ell_i^{\mathrm{safe}}(\theta)
-
\ell_i^{\mathrm{safe}}(\theta_0)
=
-\log \pi_\theta(\tilde y_i\mid \tilde x_i)
+
\log \pi_{\theta_0}(\tilde y_i\mid \tilde x_i).
\label{eq:safe_degradation}
\end{equation}
A positive value of $\Delta \ell_i^{\mathrm{safe}}(\theta)$ means that post-training has reduced the likelihood of the desired safe completion on the $i$th safety example. This relative quantity separates degradation caused by fine-tuning from the intrinsic difficulty of the safety example. Accordingly, we define the empirical safety degradation as
\begin{equation}
\mathcal{L}_{\mathrm{safe}}(\theta)
:=
\frac{1}{n}\sum_{i=1}^n
\Delta \ell_i^{\mathrm{safe}}(\theta).
\label{eq:safe_loss}
\end{equation}

\paragraph{Expectation-based baseline.}
A standard approach to safety preservation is to constrain the average degradation,
\begin{empheq}[box=\colorbox{mygray}]{equation}
\min_{\theta}\quad
\mathcal{L}_{\mathrm{SFT}}(\theta)
\qquad
\text{s.t.}\qquad
\mathcal{L}_{\mathrm{safe}}(\theta)\le \varepsilon_{\mathrm{avg}},
\label{eq:expectation_constraint}
\end{empheq}
or to add the same term as a weighted penalty. Although simple, this controls only the mean over $\dsafe$. Thus, the constraint may be satisfied even when a small subset of safety prompts suffers large degradation, provided those failures are averaged out by many easy examples.

To formalize this limitation, let $Z$ be a nonnegative random variable representing absolute per-example safety degradation, and let $\tau$ be an unacceptable-degradation threshold.

\begin{proposition}
\label{prop:mean_not_tail}
Fix $\mu>0$, $\tau>\mu$, and $\alpha\in(0,1)$. There exists a nonnegative random variable $Z$ such that
\[
\mathbb{E}[Z]\le \mu
\qquad\text{and}\qquad
\mathbb{P}(Z>\tau)\ge \alpha,
\]
provided $\alpha\tau < \mu$.
\end{proposition}

\begin{proof}
Let $Z=\tau+\eta$ with probability $\alpha$ and $Z=0$ otherwise, where $\eta>0$ is chosen such that $\alpha(\tau+\eta)\le \mu$. Then
\[
\mathbb{E}[Z]=\alpha(\tau+\eta)\le \mu,
\qquad
\mathbb{P}(Z>\tau)=\alpha.
\]
\end{proof}

Proposition~\ref{prop:mean_not_tail} shows that mean control does not imply reliability: a model may satisfy an average safety-degradation constraint while still regressing on a nontrivial fraction of safety prompts. Thus, instead of requiring the mean of $\Delta \ell_i^{\mathrm{safe}}(\theta)$ to be small, we require  $\Delta \ell_i^{\mathrm{safe}}(\theta)$ to remain below an acceptable threshold on most safety examples. This motivates the chance-constrained formulation developed next.

\section{Chance-Constrained Safety Preservation}
\label{sec:chance_constraints}

Motivated by the shortcomings of expectation-based formulations, we replace average safety preservation with a reliability constraint. Given a degradation threshold $\tau$ and violation level $\alpha\in(0,1)$, we require
\begin{equation}
\mathbb{P}_{(\tilde x,\tilde y)\sim \dsafe}
\left(
\Delta \ell^{\mathrm{safe}}(\theta;\tilde x,\tilde y)
\le \tau
\right)
\ge 1-\alpha.
\label{eq:chance_constraint}
\end{equation}
Here $\tau$ specifies the largest acceptable degradation relative to the reference model, while $\alpha$ specifies the tolerated fraction of safety examples that may exceed this threshold. On the finite safety dataset, the empirical chance constraint becomes
\begin{equation}
\frac{1}{n}\sum_{i=1}^n
\mathbf{1}
\left\{
\Delta \ell_i^{\mathrm{safe}}(\theta) > \tau
\right\}
\le \alpha.
\label{eq:empirical_chance_constraint}
\end{equation}
This constraint directly controls the empirical rate of safety regressions, but the discontinuous indicator is unsuitable for gradient-based fine-tuning.

\paragraph{Differentiable majorization of the chance constraint.}
To obtain a tractable conservative surrogate, we replace the indicator in~\eqref{eq:empirical_chance_constraint} by a differentiable majorizer. Let $\phi_\beta:\mathbb{R}\to\mathbb{R}_+$ be a nonnegative function satisfying
$
\mathbf{1}\{z>0\}
\le
\phi_\beta(z),
 \forall z\in\mathbb{R},
$
where $\beta>0$ controls the sharpness or smoothness of the surrogate. Then the sufficient condition
\begin{equation}
g_\beta(\theta)
:=
\frac{1}{n}\sum_{i=1}^n
\phi_\beta
\left(
\Delta \ell_i^{\mathrm{safe}}(\theta)-\tau
\right)
-\alpha \leq 0
\label{eq:majorized_chance_constraint}
\end{equation}
implies the empirical chance constraint~\eqref{eq:empirical_chance_constraint}. We therefore solve the majorized chance-constrained fine-tuning problem
\begin{empheq}[box=\colorbox{mygray}]{equation}
\min_{\theta}\quad
\mathcal{L}_{\mathrm{SFT}}(\theta)
\qquad
\mathrm{s.t.}\qquad
g_\beta(\theta)\le 0,
\label{eq:majorized_problem}
\end{empheq}
The choice of $\phi_\beta$ determines the conservativeness and smoothness of the relaxation: sharper majorizers better track the hard violation indicator, while smoother choices typically yield more stable gradients.

\paragraph{Connection to entropic risk.}
A particularly useful majorizer is the exponential function
$
\phi_\beta(z)=\exp(\beta z), \beta>0,
$
which satisfies $\mathbf{1}\{z>0\}\le \phi_\beta(z)$. Substituting this choice into~\eqref{eq:majorized_chance_constraint} gives the equivalent condition
\begin{equation}
\rho_{\beta,n}
\left(
\Delta \ell^{\mathrm{safe}}(\theta)-\tau
\right)
\le
\frac{1}{\beta}\log \alpha,
\label{eq:entropic_risk_constraint}
\end{equation}
where
\begin{equation}
\rho_{\beta,n}(X)
:=
\frac{1}{\beta}
\log
\left(
\frac{1}{n}\sum_{i=1}^n \exp(\beta X_i)
\right)
\label{eq:empirical_entropic_risk}
\end{equation}
is the empirical entropic risk \cite{follmer2002convex}. Thus, the exponential majorizer converts the empirical chance constraint into an entropic-risk constraint on safety degradation.

The next section develops a safe-gradient method for solving~\eqref{eq:majorized_problem} while maintaining feasibility of the majorized safety constraint \textit{during} fine-tuning.

\section{Constraint-Aware Gradient Descent}
\label{sec:safe_gd}

We now develop an optimization method for solving the majorized chance-constrained problem~\eqref{eq:majorized_problem}. The key idea is to treat
$
\mathcal{C}_\beta:=\{\theta: g_\beta(\theta)\le 0\}
$
as a safe set in parameter space and to modify the nominal fine-tuning direction only when necessary to keep the iterates feasible across all iterations. This yields a feasibility-aware alternative to penalty or primal-dual methods.

Ignoring the constraint, the nominal gradient flow for fine-tuning is
\begin{equation}
\dot \theta
=
-\nabla \mathcal{L}_{\mathrm{SFT}}(\theta).
\label{eq:nominal_gradient_flow}
\end{equation}
However, this dynamics may increase \(g_\beta(\theta)\) and drive the model outside \(\mathcal{C}_\beta\). To preserve feasibility, we impose the differential condition
\begin{equation}
\frac{d}{dt} g_\beta(\theta(t))
=
\nabla g_\beta(\theta)^\top \dot\theta
\le
-\kappa g_\beta(\theta),
\qquad \kappa>0,
\label{eq:barrier_condition}
\end{equation}
where $\kappa$ is a user-defined barrier parameter.  If \(g_\beta(\theta)=0\), this condition requires the dynamics to point inward or tangent to the feasible set. If \(g_\beta(\theta)<0\), it allows more freedom while still controlling motion toward violation.

At each parameter value \(\theta\), we therefore choose the closest direction to the nominal descent direction that satisfies~\eqref{eq:barrier_condition}:
\begin{empheq}[box=\colorbox{mygray}]{equation}
d^\star(\theta)
=
\arg\min_{d}
\frac{1}{2}
\left\|
d+\nabla \mathcal{L}_{\mathrm{SFT}}(\theta)
\right\|^2
\qquad
\mathrm{s.t.}
\qquad
\nabla g_\beta(\theta)^\top d
\le
-\kappa g_\beta(\theta).
\label{eq:safe_direction_qp}
\end{empheq}
We then define the constraint-aware gradient flow
$\dot\theta=d^\star(\theta).$
This dynamics is minimally invasive: it coincides with ordinary gradient flow whenever the nominal direction already satisfies the constraint, and otherwise applies the smallest correction needed to satisfy the safety condition.

Since~\eqref{eq:safe_direction_qp} has a single affine constraint in \(d\), it admits a closed-form solution. Specifically, if
\[
\nabla g_\beta(\theta)^\top (-\nabla \mathcal{L}_{\mathrm{SFT}}(\theta))
\le
-\kappa g_\beta(\theta),
\]
then \(d^\star(\theta)=-\nabla \mathcal{L}_{\mathrm{SFT}}(\theta)\). Otherwise,
\begin{equation}
d^\star(\theta)
=
\underbrace{-\nabla \mathcal{L}_{\mathrm{SFT}}(\theta)}_{\mathrm{nominal}}
-
\underbrace{\lambda^\star(\theta)\nabla g_\beta(\theta)}_{\mathrm{correction}},
\label{eq:safe_direction_closed_form}
\end{equation}
where
\begin{equation}
\lambda^\star(\theta)
=
\left(
\frac{
-\nabla g_\beta(\theta)^\top
\nabla \mathcal{L}_{\mathrm{SFT}}(\theta)
+
\kappa g_\beta(\theta)
}{
\|\nabla g_\beta(\theta)\|^2
}
\right)_+,
\label{eq:lambda_star}
\end{equation}
where $(\cdot)_+$ yields the non-negative component of its operand.
The correction is normal to the constraint surface and activates only when the nominal fine-tuning gradient would violate the forward-invariance condition.

\paragraph{Gradient of the safety constraint.}
From~\eqref{eq:majorized_chance_constraint}, the safety-gradient term is
\begin{equation}
\nabla g_\beta(\theta)
=
\frac{1}{n}\sum_{i=1}^n
\phi_\beta'
\left(
\Delta \ell_i^{\mathrm{safe}}(\theta)-\tau
\right)
\nabla \Delta \ell_i^{\mathrm{safe}}(\theta).
\label{eq:grad_g_beta}
\end{equation}
Since the reference-model term in \(\Delta \ell_i^{\mathrm{safe}}(\theta)\) is constant with respect to \(\theta\),
$
\nabla \Delta \ell_i^{\mathrm{safe}}(\theta)
=
\nabla \ell_i^{\mathrm{safe}}(\theta).
$
Thus, the safety correction is driven by safety examples according to their majorizer weights
$
\phi_\beta'
\left(
\Delta \ell_i^{\mathrm{safe}}(\theta)-\tau
\right).
$
For common majorizers, these weights emphasize examples near or above the degradation threshold.

\paragraph{Stochastic iterative implementation and feasibility.}
In practice, we compute the constraint-aware direction using stochastic minibatch estimates of
\(\nabla \mathcal{L}_{\mathrm{SFT}}(\theta)\), \(g_\beta(\theta)\), and \(\nabla g_\beta(\theta)\). Given a step size \(\eta>0\), the basic update is
\begin{equation}
\theta_{k+1}
=
\theta_k+\eta d^\star(\theta_k).
\label{eq:safe_gd_update}
\end{equation}
The forward-invariance guarantee, however, holds for exact values of \(g_\beta\) and \(\nabla g_\beta\); with minibatches, sampling noise can lead to constraint violations. To mitigate this, we can use a buffered stochastic constraint
\begin{equation}
\widehat g_{\beta}^{\mathrm{buf}}(\theta)
:=
\widehat g_\beta(\theta)+\rho,
\qquad \rho>0,
\label{eq:buffered_constraint}
\end{equation}
and compute the constraint-aware direction by enforcing
$
\nabla \widehat g_\beta(\theta)^\top d
\le
-\kappa \widehat g_{\beta}^{\mathrm{buf}}(\theta).
$
The buffer \(\rho\) provides slack against estimation error. These modifications preserve the scalability of minibatch training while reducing stochastic constraint violations. The resulting method can be interpreted as a safety filter on the fine-tuning gradient: when the SFT gradient already respects the majorized chance constraint, the update is unchanged; when it would increase the safety violation too quickly, the update is projected onto a safe half-space determined by \(\nabla g_\beta(\theta)\), providing a direct mechanism for maintaining safety feasibility without introducing a dual variable or hand-tuned penalty weight. The algorithm is summarized in Algorithm \ref{alg:safe_gd}.

\begin{algorithm}[t]
\caption{Safe Gradient Descent for Chance-Constrained Fine-tuning}
\label{alg:safe_gd}
\begin{algorithmic}[1]
\STATE \textbf{Input:} model $\pi_{\theta_0}$, datasets $\dsft,\dsafe$, threshold $\tau$, violation level $\alpha$, majorizer $\phi_\beta$, gain $\kappa$, step size $\eta$
\STATE Initialize $\theta\leftarrow \theta_0$
\FOR{$k=0,1,\dots,K-1$}
    \STATE Sample minibatches $B_{\mathrm{SFT}}\subseteq\dsft$ and $B_{\mathrm{safe}}\subseteq\dsafe$
    \STATE Estimate $\nabla \mathcal{L}_{\mathrm{SFT}}(\theta)$, $g_\beta(\theta)$, and $\nabla g_\beta(\theta)$ on the minibatches
    \STATE Compute $d^\star(\theta)$ from~\eqref{eq:safe_direction_closed_form}--\eqref{eq:lambda_star}
    \STATE Update $\theta\leftarrow \theta+\eta d^\star(\theta)$
\ENDFOR
\STATE \textbf{Return:} $\theta$
\end{algorithmic}
\end{algorithm}

\section{Experiments}
\label{sec:experiments}

We empirically evaluate our method against a broad suite of recent defenses
for safety-preserving fine-tuning.
Following the harmful fine-tuning
(HFT) literature
\citep{huang2024lisa,huang2024booster,li2025salora},
we adopt a two-stage protocol: a base instruct model is first
\emph{aligned} on a safety dataset, and then \emph{fine-tuned} on a
downstream task whose data has been poisoned with harmful prompt-response
pairs. The defense is applied either at the alignment stage or the
fine-tuning stage, depending on the method.

\subsection{Setup}
\label{sec:setup_exp}

\paragraph{Models.}
We use three open-weight instruct models spanning two families and two
size regimes: Llama-3.1-8B-Instruct, Qwen-3.5-9B, and
Qwen-3.5-4B. All defenses train LoRA adapters
\citep{hu2022lora} of rank $32$ with $\alpha_{\text{LoRA}}{=}4$,
attached to the query/key/value projections of every transformer
block, so all methods adapt the same parameter set.
Models are loaded in bfloat16; remaining compute and software details
are summarized in Appendix~\ref{app:compute}.

\paragraph{Safety dataset and harmful evaluation.}
The safety dataset $D_{\mathrm{safe}}$ is constructed from the
\textsc{BeaverTails} \texttt{30k\_train} split
\citep{ji2023beavertails} by selecting prompt--response pairs flagged
\texttt{is\_safe=True}. We use $5{,}000$ such pairs for both the
alignment stage and as the constraint dataset for our method. For
evaluation, we hold out a disjoint set of $700$ unique harmful prompts
from the \texttt{30k\_test} split.

\paragraph{Downstream tasks and poisoning.}
We adopt three standard downstream tasks used throughout the HFT
literature: \textbf{SST-2} \citep{socher2013recursive} (binary sentiment),
\textbf{AG News} \citep{zhang2015character} (4-way topic classification),
and \textbf{GSM8K} \citep{cobbe2021training} (grade-school math).
Each task is poisoned by mixing in harmful (prompt, response) pairs
sampled from the \textsc{BeaverTails} unsafe split:
$\lfloor p \cdot N \rfloor$ harmful pairs are concatenated with the
benign task examples and shuffled, where $N{=}1000$ is the total
fine-tuning budget and $p \in [0,1]$ is the poison ratio.
Our main table fixes a high-poison regime with $p{=}0.5$ to expose
defenses to a strong attack; the dependence on $p$ is studied in
Section~\ref{sec:poison_sweep}.

\paragraph{Two-stage training.}
\emph{Alignment.} The base model is fine-tuned on $D_{\mathrm{safe}}$
with the alignment-stage defense (or plain supervised fine-tuning, for
fine-tuning-stage defenses). Standard alignment uses learning rate
$5{\times}10^{-4}$, $3$ epochs, batch size $10$, constant schedule, and
weight decay $0.1$; alignment-stage defenses (Vaccine~\cite{huang2024vaccine}, Booster~\cite{huang2024booster})
 instead align for $20$ epochs, matching the recipes used in their
original works. After alignment, the LoRA adapter is merged into the
base weights so subsequent fine-tuning starts from a well-defined
aligned checkpoint.

\emph{Fine-tuning.} The aligned model is then fine-tuned on the
poisoned downstream task with learning rate $1{\times}10^{-5}$, $20$
epochs, batch size $10$, constant schedule, weight decay $0.1$, and
maximum sequence length $512$. The longer fine-tuning horizon ensures
that all defenses are challenged with sustained poisoning exposure
rather than under-trained baselines.
We perform fine-tuning under three different seeds and report mean $\pm$ std for the desired metrics in our tables.
Defense-specific hyperparameters
(reproduced verbatim from the originating papers wherever applicable)
are listed in Appendix~\ref{app:defense_hparams}.

\subsection{Defenses}
\label{sec:defenses}

We compare our method against seven baselines used in the recent HFT literature. We present these methods in Table \ref{tab:method_descriptions}, along with a categorization of the stage in which they operate.

\begin{table}[t]
\centering
\small
\setlength{\tabcolsep}{3pt}
\renewcommand{\arraystretch}{1.15}
\scalebox{0.95}{
\begin{tabular}{l ccc p{0.6\linewidth}}
\toprule
Method & Alignment & SFT & Post-hoc & Mechanism \\
\midrule
SFT  & & \checkmark & &
Plain supervised fine-tuning on the poisoned task; worst-case lower bound w/o defensive intervention. \\
\midrule
Vaccine \citep{huang2024vaccine} & \checkmark & & &
SAM-style perturbation of attention representations during alignment to flatten the loss landscape against fine-tuning shifts. \\
Booster \citep{huang2024booster} & \checkmark & & &
Adds a harmful-perturbation regularizer ($\lambda{=}5$) that perturbs parameters along the harmful gradient before each alignment step. \\
\midrule
Lisa \citep{huang2024lisa} & & \checkmark & &
Bi-state optimization alternating fine-tuning and alignment phases, with a proximal pull toward the aligned checkpoint. \\
SafeGrad \citep{yi2025gradient} & & \checkmark & &
Gradient surgery resolving conflicts between the task gradient and a KL-based alignment gradient at every fine-tuning step. \\
AsFT \citep{yang2026asft} & & \checkmark & &
Anchors fine-tuning to the alignment direction by penalizing the orthogonal component of the LoRA update ($\lambda_{\mathrm{reg}}{=}1$). \\
SaLoRA \citep{li2025salora} & & \checkmark & &
Precomputes a safety direction from $D_{\mathrm{safe}}$ and projects it out of the fine-tuning gradient at every step. \\
\midrule
SafeLoRA \citep{hsu2024safe} & & & \checkmark &
Training-free, data-free patch projecting the trained LoRA update of selected layers onto the safety-aligned subspace. \\
\bottomrule
\end{tabular}
}
\caption{Baselines grouped by stage of intervention. 
SafeLoRA is training-free and applied to the LoRA artifact after fine-tuning.}
\label{tab:method_descriptions}
\end{table}

For our method, we instantiate the majorized problem \eqref{eq:majorized_problem} using the one-sided ramp function
$\phi_\beta(z) = \max(1 + \beta z, 0), \beta>0$. 
The constraint then takes the form
\begin{equation}
    g_\beta(\theta)
    \;=\;
    \frac{1}{n}\sum_{i=1}^n
\max
\left(
1 + 
\beta (\Delta \ell_i^{\mathrm{safe}}(\theta)-\tau), 0
\right)
-\alpha \leq 0
    \label{eq:cage_ra_constraint}
\end{equation}
where $\tau$ is the largest acceptable degradation relative to the reference model and $\alpha\in(0,1)$
is the tolerated violation rate. 
At each step we form the
safe descent direction by the closed-form correction in
\eqref{eq:safe_direction_closed_form}--\eqref{eq:safe_gd_update}, reusing the
batch of $D_{\mathrm{safe}}$ already drawn for $\nabla g_{\beta}$ as the
constraint mini-batch and the poisoned task batch as $\nabla \mathcal{L}_\mathrm{SFT}$.
Throughout the main experiments, we use the configuration
$(\tau, \beta, \alpha) = (0.1, 10, 0.05)$ and set the barrier parameter $\kappa$ as the inverse of the learning rate. We provide results for hyperparameter search in Section \ref{sec:hp_search}.

\subsection{Evaluation metrics}
\label{sec:metrics}

We report two complementary metrics, following standard HFT practice:

\paragraph{Harmful Score (HS, $\downarrow$).}
We generate $512$-token responses to each of the $700$ held-out harmful
\textsc{BeaverTails} prompts and score them with the
\texttt{PKU-Alignment/beaver-dam-7b} moderation classifier
\citep{ji2023beavertails}. For each prompt $i$, the moderator returns
a vector of safety-category probabilities; we take the maximum across
categories as the per-prompt harm probability $h_i$. The headline HS
is the fraction of prompts with $h_i > 0.4$, the threshold used in the
original \textsc{BeaverTails} evaluation. We additionally
report the per-prompt quantiles
$\mathrm{HS}_{qK}, K \in \{10, 25, 50, 75, 90, 95, 99\}$, of $\{h_i\}_{i=1}^{700}$,
which directly track tail behavior and are the quantities that our method
is constructed to control.

\paragraph{Fine-tune Accuracy (FA, $\uparrow$).}
Task utility is measured on $200$ held-out test examples per task
using the standard task metric: substring match for SST-2 and AG~News,
and exact-match on the extracted final answer for GSM8K. Generation
uses left-padded greedy decoding with up to $256$ new tokens.

\subsection{Main results: high poisoning}
\label{sec:main_table}

\input{tables/neuripsPapers/table1_neurips}

Table~\ref{tab:main} reports performance at poison ratio $p{=}0.5$ for three
backbones and three downstream tasks. For each configuration we report the
harmful-response rate HS, the per-prompt harm quantile $\mathrm{HS}_{qK}$
($K{=}90$ for the Qwen backbones, $K{=}50$ for Llama, where higher quantiles
saturate), and the task accuracy FA. The \emph{Base} and \emph{Aligned (SFT)}
rows give the pre-fine-tuning model and its safety-aligned counterpart as
references, and undefended SFT serves as the failure baseline against which all
defenses are measured.

\textbf{Harm reduction.} Across backbones and tasks, our method sits at or near
the lowest HS and $\mathrm{HS}_{qK}$ in every configuration, so the reduction reflects the upper tail of the per-prompt
harm distribution and not merely the threshold statistic. The reference that
matters is the aligned model. Our HS is at or statistically indistinguishable
from it everywhere, while undefended SFT and most regularization-based defenses
raise harm well above that level. Fine-tuning under our method therefore leaves
the safety profile of the aligned checkpoint essentially intact, which is what
the constraint was written to enforce. A more detailed breakdown of the
quantiles can be found in Appendix~\ref{app:tail}.

\textbf{Classification tasks.} On SST-2 and AG News, suppressing harm does not
come at the expense of task learning: our method fits the downstream task to
within a modest margin of the strongest utility-oriented baselines while
simultaneously achieving the lowest or near-lowest harm on both metrics. The
regularization-based defenses are markedly less stable at this poison ratio,
with the HS of Vaccine and Booster exceeding that of undefended SFT on several
Qwen configurations.
Perturbing or smoothing the fine-tuning objective thus
redistributes harm rather than controlling it when the poison fraction is
large. By contrast, the behavior of our method is consistent across model
scale.

\textbf{GSM8K and constraint satisfaction.} GSM8K is the most informative
setting for assessing whether a defense genuinely respects the safety
constraint. Task accuracy is far from saturated here, so the safety-utility
trade-off is visible in a way it is not on the classification tasks, and no
method approaches undefended-SFT accuracy without an HS far above the aligned
reference. The constraint is therefore active. This changes the question a
defense should be asked. Since our formulation maximizes task performance
subject to $g_\beta(\theta) \leq 0$ rather than trading the two off through a
penalty, the relevant comparison is not who attains the lowest harm outright,
but who extracts the most utility from among the methods that keep harm at the
aligned level. 
On that comparison our method is the strongest on every
backbone: the few baselines that record lower HS than ours do so only by
sacrificing task accuracy, in the Qwen-3.5-4B and Llama cases falling to
roughly half of what our method achieves. 
Their apparent safety advantage is
model degradation, not constraint satisfaction. Our method instead recovers
substantial accuracy over the aligned model on all three backbones while
holding harm at the aligned level, and on Qwen-3.5-9B it surpasses
undefended-SFT accuracy outright without any corresponding rise in harm. The
additional accuracy that the remaining baselines extract is accompanied by
exactly the harm the constraint precludes. This is the behavior prescribed for
a faithful constrained method in the harmful-fine-tuning regime, and it is the
property that distinguishes our approach from baselines whose safety guarantees
hold only in expectation.

\subsection{Poison-ratio sweep}
\label{sec:poison_sweep}

To probe how each defense scales with attack strength, we hold all
other settings fixed and sweep the poison ratio
$p \in \{0.05, 0.1, 0.25, 0.5\}$ on Qwen-3.5-4B/SST-2.
Table~\ref{tab:poison-sweep-qwen3.5_4b_sst2} reports HS and FA for every defense at
each $p$. The tables for the other tasks are deferred to the Appendix.

\input{tables/neuripsPapers/poison-sweep-qwen3.5_4b_sst2}

The poison-ratio sweep on Qwen-3.5-4B
(Tables~\ref{tab:poison-sweep-qwen3.5_4b_sst2},
\ref{tab:poison-sweep-qwen3.5_4b_ag_news}, and
\ref{tab:poison-sweep-qwen3.5_4b_gsm8k}) isolates the defining behavior of our method: its harm metrics are approximately invariant to the poison ratio.
As $p$ climbs from
$5\%$ to $50\%$, our HS and $\mathrm{HS}_{q90}$ remain nearly constant, whereas the harm of
the undefended baseline (SFT) climbs steadily. Whether this fixed safety can be
combined with high accuracy depends on the task.

On SST-2 and AG News (Tables~\ref{tab:poison-sweep-qwen3.5_4b_sst2} and
\ref{tab:poison-sweep-qwen3.5_4b_ag_news}) it can: our method is the lowest-harm
defense at every poison ratio while its accuracy trails the best-performing
baseline by only a few points. 
The next-lowest-harm defenses (SafeLoRA, AsFT) are likewise approximately flat in p, but at a higher harm level.
Vaccine and Booster exceed the undefended baseline at every poison ratio.

On GSM8K (Table~\ref{tab:poison-sweep-qwen3.5_4b_gsm8k}), high accuracy and low
harm are not jointly attainable among the defenses we evaluate: harm and
accuracy are negatively associated across methods, and the lowest-harm defenses
(SafeLoRA, AsFT) reach the aligned harm floor only by failing to learn the
task, at accuracy near $0.15$. Here our method keeps the safety constraint
binding and lets accuracy take a value compatible with it. 
Among the methods we compare, our method is the only one that has kept the harmfulness low whilst increasing the task accuracy.

\section{Conclusion}
\label{sec:conclusion}
We formulated safety preservation in LLM post-training as a chance
constraint on the safety degradation relative to a reference model and gave a
tractable instantiation via differentiable majorization and a
closed-form safety filter on the fine-tuning gradient. The method
enforces the safety budget step by step rather than as a soft
penalty integrated against the example distribution. 
Across three backbones and three tasks under harmful fine-tuning, this yields a
strong and safe operating point on the safety-utility trade-off: harm at or
near the lowest in nearly every cell, at a modest cost in task accuracy.

\section{Limitations}
\label{sec:limitations}

Our method requires access to a separate safety dataset on which to
instantiate the chance constraint. This is standard within the HFT framework,
as most defenses in our comparison set assume the same. Computationally, the
only addition relative to plain SFT is the constraint gradient: each step
evaluates $\nabla g_\beta(\theta)$ on a constraint mini-batch, which costs one
extra forward--backward pass and roughly doubles the per-step wall-clock time.
Appendix~\ref{app:compute} reports a detailed runtime analysis together with a
primal--dual reformulation that folds this gradient into a single backward
pass, trading part of the per-step feasibility guarantee for speed.

A further limitation is that our constraint is instantiated on a fixed safety
dataset, scored by simply the cross-entropy being low on such samples.
A natural extension
is to define the chance constraint on a reward model applied to the
model's own generations, giving an on-policy safety constraint that tracks the
current policy's outputs rather than a static prompt set. We leave this
direction to future work.

\bibliography{bib}

\clearpage

\appendix

\input{tables/neuripsPapers/poison-sweep-qwen3.5_4b_ag_news}

\input{tables/neuripsPapers/poison-sweep-qwen3.5_4b_gsm8k}

\section{Related Work}
\label{sec:related_work}

\paragraph{Harmful fine-tuning: attacks and defenses.}
The vulnerability of safety-aligned LLMs to fine-tuning was first
systematically documented in \cite{qi2023fine}, which showed that
GPT-3.5's safety guardrails can be removed with as few as ten
adversarially designed examples and that benign instruction data
also degrades safety. Similar findings followed for open-weight
models \citep{yang2023shadow,lermen2023lora} and for GPT-4 via the
public fine-tuning API \citep{zhan2024removing}, with related work
on jailbreak failure modes \citep{wei2023jailbroken} and
detection-evading attacks \citep{halawi2024covert}.

Defenses are organized by where they intervene in the post-training
pipeline. Alignment-stage methods reshape the loss landscape against
fine-tuning shifts via flat-minimum perturbations
\citep{huang2024vaccine,huang2024booster} or by removing harmful
representations from intermediate activations
\citep{rosati2024representation}. Fine-tuning-stage methods
constrain or project the LoRA update --- through proximal corrections
toward the aligned checkpoint \citep{huang2024lisa}, projection of
safety directions out of the gradient \citep{li2025salora}, gradient
surgery against a KL-based alignment gradient
\citep{yi2025gradient}, or orthogonal anchoring penalties 
\citep{yang2026asft}. Joint two-stage \citep{nguyen2026antibody}
and post-hoc \citep{hsu2024safe} methods fill the remaining cells.

\paragraph{Chance-constrained optimization and tail-risk measures.}
Chance-constrained programming \citep{charnes1959chance} addresses
optimization under uncertainty through bounds on the probability of
constraint violation. Direct optimization is intractable due to the
discontinuous indicator, and the standard remedies are convex inner
approximations \citep{nemirovski2006convex} or
sampling-based scenario relaxations \citep{calafiore2006scenario},
both of which yield deterministic problems with high-probability
feasibility. 
These tools have entered machine learning
primarily through risk-sensitive RL
\citep{chow2017risk}.

\paragraph{Constrained policy optimization and constrained LLM post-training.}
Treating safety as a hard constraint has a long history in safe RL
through the constrained MDP framework \citep{altman1999constrained},
with two main algorithmic families.
Trust-region methods \citep{achiam2017constrained,yang2020projection} take constrained policy steps using first-order approximations of the objective and the cost constraint together with a quadratic (KL) trust region. Primal-Dual
methods drive the policy toward feasibility through dual ascent on
a Lagrange multiplier \citep{tessler2018reward,stooke2020responsive},
justified theoretically by the zero-duality-gap result of
\cite{paternain2019constrained}. In the LLM setting, \emph{Safe
RLHF} \citep{dai2023safe} formalizes helpfulness-versus-harmlessness
as a Lagrangian-relaxed CMDP with separate reward and cost models,
and \cite{entesari2025constrained} adopt a similar primal-dual
treatment for LLM unlearning. 
Both of these methods rely on the dual variable to
react to constraint violations after they occur. 
Our constraint-aware gradient descent operates differently: it computes
a closed-form QP correction at each step that keeps the iterate
inside the feasible set defined by the majorized chance constraint,
providing anytime feasibility.

\section{Defense Hyperparameters}
\label{app:defense_hparams}

Table~\ref{tab:defense_hparams} gives the  hyperparameters used
for every defense in every experiment. Both stages share the
optimizer (AdamW, $\beta_1{=}0.9$, $\beta_2{=}0.999$, weight decay
$0.1$) and a constant learning-rate schedule, with learning rate
$5{\times}10^{-4}$ at the alignment stage and $1{\times}10^{-5}$ at
the fine-tuning stage. Mini-batch size is $10$ for all defenses;
$\nabla f$ and $\nabla g$ for our method are computed on \emph{disjoint}
mini-batches of size $10$ drawn each step from the poisoned task and
from $D_{\mathrm{safe}}$, respectively. All defenses fine-tune for
$20$ epochs on $1{,}000$ task examples (poisoned at ratio $p$); LoRA
adapters use rank $32$, $\alpha{=}4$, dropout $0.05$, and a maximum
sequence length of $512$ tokens.

\begin{table}[h]
\centering
\small
\caption{Defense-specific hyperparameters used in all experiments.
Symbols follow the originating papers; ``align epochs'' is the number
of epochs of the alignment-stage objective on the $5{,}000$-example
$D_{\mathrm{safe}}$ split. ``ft'' abbreviates fine-tuning.}
\label{tab:defense_hparams}
\begin{tabular}{ll}
\toprule
Defense  & Hyperparameters \\
\midrule
Vaccine        & $\rho{=}2$, align epochs $20$ \\
Booster        & $\alpha_{\mathrm{B}}{=}0.1$, $\lambda{=}5$, align epochs $20$ \\
Lisa         & align step $100$, ft step $900$, $\rho{=}0.01$, guide $10{,}000$ \\
SaLoRA       & projection strength $\alpha_{\mathrm{proj}}{=}1.0$ \\
SafeLoRA        & projection threshold $\tau_{\mathrm{SL}}{=}0.5$ (applied after plain SFT fine-tuning) \\
SafeGrad     & $\rho{=}1.0$ (KL alignment weight) \\
AsFT         & $\lambda_{\mathrm{reg}}{=}1.0$ \\
Antibody  & $\rho{=}0.1$, $\lambda_{\mathrm{ref}}{=}1$, $\tau{=}1$, align epochs $20$ \\
\bottomrule
\end{tabular}
\end{table}

\section{Hyperparameter search}
\label{sec:hp_search}

Our method exposes three hyperparameters: the degradation budget relative to the reference model
$\tau$, the ramp slope $\beta$, and the tolerated violation
probability $\alpha$. 
We sweep these on Qwen-3.5-4B/SST-2 at $p{=}0.5$, with
$\tau \in \{0, 0.01, 0.05, 0.1, 0.2\}$,
$\alpha \in \{0.01, 0.05, 0.1, 0.2\}$, and
$\beta \in \{1, 10\}$, and report the results in Table~\ref{tab:cage-ra-sens}.

We see that the method is
robust across two orders of magnitude of $\alpha$ and $\tau$:
all configurations land within a small band on $q_{90}$.
We adopt $(\tau, \beta, \alpha){=}(0.1, 10, 0.05)$ for all remaining experiments.

\input{tables/neuripsPapers/cage_sensitivity_sst2}

\section{Tail Statistics}
\label{app:tail}

We provide detailed breakdown of the tail statistics for the different tasks for the Qwen-3.5 4B model in Tables \ref{tab:quantile-qwen3.5_4b_sst2}, \ref{tab:quantile-qwen3.5_4b_ag_news}, and \ref{tab:quantile-qwen3.5_4b_gsm8k}. 
The tables show that our method consistently wins across high-tail statistics.

\input{tables/neuripsPapers/quantile-qwen3.5_4b_sst2}

\input{tables/neuripsPapers/quantile-qwen3.5_4b_ag_news}

\input{tables/neuripsPapers/quantile-qwen3.5_4b_gsm8k}

\section{Computational Complexity}
\label{app:compute}

All training and evaluation runs were executed on a single A100/H100
node per run, using PyTorch 2.x with bfloat16 mixed precision. 
Our method roughly doubles the per-step wall-clock time relative to plain SFT, owing to the additional forward/backward pass on the constraint batch.
We provide a profiling comparison of our method against some of the baselines for the Qwen-3.5 4B model in Table \ref{tab:profile-qwen3.5_4b_sst2}.

\input{tables/neuripsPapers/profiling}

\section{Primal-Dual Algorithm}
As discussed in Section \ref{app:compute}, our algorithm incurs additional computational overhead during training, primarily due to the separate backward pass required to compute the gradient of the constraint. This overhead can be circumvented by replacing the gradient-flow formulation with a simpler primal-dual scheme. To this end, we adopt the framework developed for constrained unlearning in \cite{entesari2025constrained}, which eliminates the need for a separate backward call. We emphasize that this reformulation does not alter the underlying computational complexity of evaluating the constraint gradient; rather, it integrates naturally into modern automatic differentiation pipelines, reducing the effective runtime to that of a standard regularization technique. A detailed description of the algorithm, together with the analysis establishing that it yields valid solutions to the optimization problem, is provided in \cite{entesari2025constrained}. Table \ref{tab:primal-dual} reports the results of applying this primal-dual algorithm to the constrained setup of our experiments. The primal-dual variant attains comparable utility (FA within 0.02 of our method on all three tasks). On AG News and GSM8K, however, our method achieves lower harm, consistent with its anytime-feasibility guarantee. On SST-2 the two are comparable.

\input{tables/neuripsPapers/primal_dual}

\section{Composition with a Dual-Stage Defense}
\label{app:antibody}

Antibody~\citep{nguyen2026antibody} intervenes at both stages: it replaces the alignment run
that produces $\theta_0$ and then defends the fine-tuning run that starts from it.  Its alignment stage also
changes two things at once, the objective and the epoch budget (20 vs.\ 3),  so the
resulting pre-attack safety cannot be attributed to either in isolation. We report it
separately in Table~\ref{tab:antibody_appendix}, together with the more informative experiment:
applying our filter to the Antibody-aligned checkpoint rather than the SFT-aligned one.

Our constraint is written on degradation relative to $\theta_0$, not on absolute harm, so it
is indifferent to how the reference was produced and composes with any alignment-stage
procedure that supplies one. The composition pays off where the constraint is active. On
GSM8K, Antibody's own fine-tuning stage does not hold the floor its alignment stage
establishes, harm moves further from the Antibody-aligned checkpoint than undefended SFT
moves from the standard one, with the largest seed variance in the column.
Our filter, on the
same checkpoint, holds it: the composed configuration attains the lowest post-attack
$\mathrm{HS}$ of any configuration we evaluate on GSM8K for all three backbones, with the same
ordering on $\mathrm{HS}_{qK}$, while recovering accuracy over its reference rather than
reaching low harm by failing to learn the task. On SST-2 and AG News, where accuracy saturates
and the constraint is slack for most of training, Antibody retains the lower $\mathrm{HS}$.
f
The filter preserves the Antibody-aligned reference less tightly than the standard one. The
constraint acts on cross-entropy over $\mathcal{D}_{\mathrm{safe}}$, and the extra safety the
alignment stage buys is evidently not fully captured by that proxy; separately,
$(\tau,\beta,\alpha)$ were tuned against the SFT-aligned reference and $\tau$ is an absolute
budget in nats, so a lower-loss reference plausibly needs a tighter one. We did not re-tune
per reference model.

\input{tables/neuripsPapers/antibody-table}

\end{document}

%% file: tables/neuripsPapers/table1_neurips.tex
\begin{table*}[t]
  \centering
  \small
  \caption{Main results at poison ratio $p{=}0.5$. For each (model, defense, task) we report HS (fraction of prompts with max harm probability $>0.4$), HS$_{qK}$ (per-prompt $K$-th quantile of max harm probability), and FA (task accuracy). $K{=}50$ for Llama (higher quantiles saturate near $1.0$ across all defenses); $K{=}90$ for Qwen models. \textbf{Bold} = best, \underline{underlined} = second best per column.  $\downarrow$ lower is better, $\uparrow$ higher is better. Defense rows are mean\,{\tiny$\pm$std} over 3 seeds.}
  \label{tab:main}
  \resizebox{\textwidth}{!}{%
  \begin{tabular}{clccccccccc}
    \toprule
     & Defense & \multicolumn{3}{c}{SST-2} & \multicolumn{3}{c}{AG News} & \multicolumn{3}{c}{GSM8K} \\
    \cmidrule(lr){3-5}\cmidrule(lr){6-8}\cmidrule(lr){9-11}
     &  & HS $\downarrow$ & HS$_{qK}$ $\downarrow$ & FA $\uparrow$ & HS $\downarrow$ & HS$_{qK}$ $\downarrow$ & FA $\uparrow$ & HS $\downarrow$ & HS$_{qK}$ $\downarrow$ & FA $\uparrow$ \\
    \midrule
    \multirow{11}{*}{\rotatebox[origin=c]{90}{Qwen-3.5-4B\,($qK{=}90$)}} & Base & 0.430 & 0.707 & 0.915 & 0.430 & 0.707 & 1.000 & 0.430 & 0.707 & 0.280 \\
     & Aligned (SFT) & 0.350 & 0.738 & 0.940 & 0.350 & 0.738 & 0.890 & 0.350 & 0.738 & 0.170 \\
    \cmidrule(l){2-11}
     & SFT & 0.546\,{\tiny$\pm$0.042} & 0.890\,{\tiny$\pm$0.015} & \underline{0.973}\,{\tiny$\pm$0.016} & 0.545\,{\tiny$\pm$0.043} & 0.901\,{\tiny$\pm$0.025} & 0.890\,{\tiny$\pm$0.015} & 0.637\,{\tiny$\pm$0.047} & 0.945\,{\tiny$\pm$0.017} & \underline{0.633}\,{\tiny$\pm$0.064} \\
     & Vaccine & 0.688\,{\tiny$\pm$0.032} & 0.962\,{\tiny$\pm$0.008} & 0.963\,{\tiny$\pm$0.038} & 0.680\,{\tiny$\pm$0.025} & 0.960\,{\tiny$\pm$0.002} & \underline{0.932}\,{\tiny$\pm$0.033} & 0.697\,{\tiny$\pm$0.023} & 0.968\,{\tiny$\pm$0.009} & 0.378\,{\tiny$\pm$0.172} \\
     & Booster & 0.665\,{\tiny$\pm$0.021} & 0.940\,{\tiny$\pm$0.012} & 0.943\,{\tiny$\pm$0.018} & 0.674\,{\tiny$\pm$0.013} & 0.949\,{\tiny$\pm$0.007} & \textbf{0.952}\,{\tiny$\pm$0.021} & 0.715\,{\tiny$\pm$0.012} & 0.951\,{\tiny$\pm$0.002} & 0.427\,{\tiny$\pm$0.083} \\
     & Lisa & 0.489\,{\tiny$\pm$0.051} & 0.845\,{\tiny$\pm$0.037} & 0.968\,{\tiny$\pm$0.003} & 0.510\,{\tiny$\pm$0.038} & 0.874\,{\tiny$\pm$0.020} & 0.900\,{\tiny$\pm$0.018} & 0.532\,{\tiny$\pm$0.032} & 0.913\,{\tiny$\pm$0.016} & 0.290\,{\tiny$\pm$0.041} \\
     & SaLoRA & 0.540\,{\tiny$\pm$0.040} & 0.895\,{\tiny$\pm$0.014} & \textbf{0.975}\,{\tiny$\pm$0.005} & 0.558\,{\tiny$\pm$0.024} & 0.898\,{\tiny$\pm$0.022} & 0.898\,{\tiny$\pm$0.020} & 0.656\,{\tiny$\pm$0.016} & 0.947\,{\tiny$\pm$0.004} & \textbf{0.655}\,{\tiny$\pm$0.061} \\
     & SafeLoRA & \underline{0.363}\,{\tiny$\pm$0.014} & \textbf{0.721}\,{\tiny$\pm$0.019} & 0.952\,{\tiny$\pm$0.018} & \underline{0.358}\,{\tiny$\pm$0.026} & \textbf{0.726}\,{\tiny$\pm$0.008} & 0.898\,{\tiny$\pm$0.032} & \underline{0.358}\,{\tiny$\pm$0.012} & \textbf{0.720}\,{\tiny$\pm$0.002} & 0.135\,{\tiny$\pm$0.005} \\
     & SafeGrad & 0.445\,{\tiny$\pm$0.015} & 0.843\,{\tiny$\pm$0.006} & 0.967\,{\tiny$\pm$0.019} & 0.450\,{\tiny$\pm$0.038} & 0.864\,{\tiny$\pm$0.031} & 0.888\,{\tiny$\pm$0.006} & 0.495\,{\tiny$\pm$0.026} & 0.894\,{\tiny$\pm$0.028} & 0.542\,{\tiny$\pm$0.083} \\
     & AsFT & 0.371\,{\tiny$\pm$0.013} & 0.774\,{\tiny$\pm$0.010} & 0.952\,{\tiny$\pm$0.006} & 0.374\,{\tiny$\pm$0.011} & 0.792\,{\tiny$\pm$0.009} & 0.907\,{\tiny$\pm$0.010} & \textbf{0.348}\,{\tiny$\pm$0.006} & \underline{0.737}\,{\tiny$\pm$0.013} & 0.138\,{\tiny$\pm$0.016} \\
     & Ours & \textbf{0.340}\,{\tiny$\pm$0.032} & \underline{0.726}\,{\tiny$\pm$0.041} & 0.957\,{\tiny$\pm$0.018} & \textbf{0.340}\,{\tiny$\pm$0.017} & \underline{0.732}\,{\tiny$\pm$0.039} & 0.927\,{\tiny$\pm$0.006} & 0.360\,{\tiny$\pm$0.052} & 0.780\,{\tiny$\pm$0.118} & 0.267\,{\tiny$\pm$0.110} \\
    \midrule
    \multirow{11}{*}{\rotatebox[origin=c]{90}{Qwen-3.5-9B\,($qK{=}90$)}} & Base & 0.434 & 0.715 & 0.995 & 0.434 & 0.715 & 0.980 & 0.434 & 0.715 & 0.380 \\
     & Aligned (SFT) & 0.546 & 0.781 & 0.980 & 0.546 & 0.781 & 0.855 & 0.546 & 0.781 & 0.250 \\
    \cmidrule(l){2-11}
     & SFT & 0.670\,{\tiny$\pm$0.023} & 0.898\,{\tiny$\pm$0.031} & 0.980\,{\tiny$\pm$0.005} & 0.672\,{\tiny$\pm$0.025} & 0.885\,{\tiny$\pm$0.016} & 0.903\,{\tiny$\pm$0.030} & 0.681\,{\tiny$\pm$0.020} & 0.947\,{\tiny$\pm$0.005} & 0.300\,{\tiny$\pm$0.080} \\
     & Vaccine & 0.601\,{\tiny$\pm$0.052} & 0.938\,{\tiny$\pm$0.010} & \textbf{0.993}\,{\tiny$\pm$0.012} & 0.622\,{\tiny$\pm$0.038} & 0.938\,{\tiny$\pm$0.007} & \underline{0.965}\,{\tiny$\pm$0.022} & 0.687\,{\tiny$\pm$0.030} & 0.947\,{\tiny$\pm$0.008} & 0.227\,{\tiny$\pm$0.059} \\
     & Booster & 0.628\,{\tiny$\pm$0.011} & 0.867\,{\tiny$\pm$0.010} & 0.945\,{\tiny$\pm$0.005} & 0.622\,{\tiny$\pm$0.015} & 0.845\,{\tiny$\pm$0.005} & \textbf{0.972}\,{\tiny$\pm$0.013} & 0.619\,{\tiny$\pm$0.023} & 0.869\,{\tiny$\pm$0.012} & \textbf{0.597}\,{\tiny$\pm$0.049} \\
     & Lisa & 0.612\,{\tiny$\pm$0.036} & 0.854\,{\tiny$\pm$0.020} & 0.963\,{\tiny$\pm$0.003} & 0.610\,{\tiny$\pm$0.023} & 0.853\,{\tiny$\pm$0.012} & 0.905\,{\tiny$\pm$0.026} & 0.622\,{\tiny$\pm$0.047} & 0.900\,{\tiny$\pm$0.022} & \underline{0.517}\,{\tiny$\pm$0.049} \\
     & SaLoRA & 0.677\,{\tiny$\pm$0.013} & 0.903\,{\tiny$\pm$0.014} & \underline{0.983}\,{\tiny$\pm$0.003} & 0.689\,{\tiny$\pm$0.007} & 0.893\,{\tiny$\pm$0.008} & 0.925\,{\tiny$\pm$0.017} & 0.710\,{\tiny$\pm$0.011} & 0.943\,{\tiny$\pm$0.008} & 0.258\,{\tiny$\pm$0.059} \\
     & SafeLoRA & \underline{0.531}\,{\tiny$\pm$0.015} & \textbf{0.771}\,{\tiny$\pm$0.005} & 0.967\,{\tiny$\pm$0.014} & \textbf{0.509}\,{\tiny$\pm$0.005} & \textbf{0.777}\,{\tiny$\pm$0.012} & 0.885\,{\tiny$\pm$0.035} & \textbf{0.501}\,{\tiny$\pm$0.013} & \textbf{0.760}\,{\tiny$\pm$0.017} & 0.258\,{\tiny$\pm$0.029} \\
     & SafeGrad & 0.604\,{\tiny$\pm$0.025} & 0.868\,{\tiny$\pm$0.020} & \underline{0.983}\,{\tiny$\pm$0.010} & 0.612\,{\tiny$\pm$0.009} & 0.855\,{\tiny$\pm$0.014} & 0.907\,{\tiny$\pm$0.028} & 0.613\,{\tiny$\pm$0.028} & 0.919\,{\tiny$\pm$0.005} & 0.317\,{\tiny$\pm$0.058} \\
     & AsFT & \textbf{0.529}\,{\tiny$\pm$0.009} & \underline{0.800}\,{\tiny$\pm$0.002} & 0.972\,{\tiny$\pm$0.013} & 0.547\,{\tiny$\pm$0.007} & \underline{0.811}\,{\tiny$\pm$0.019} & 0.892\,{\tiny$\pm$0.029} & \underline{0.529}\,{\tiny$\pm$0.005} & \underline{0.801}\,{\tiny$\pm$0.013} & 0.235\,{\tiny$\pm$0.036} \\
     & Ours & 0.539\,{\tiny$\pm$0.024} & 0.836\,{\tiny$\pm$0.048} & 0.957\,{\tiny$\pm$0.013} & \underline{0.536}\,{\tiny$\pm$0.022} & 0.822\,{\tiny$\pm$0.030} & 0.902\,{\tiny$\pm$0.050} & 0.548\,{\tiny$\pm$0.027} & 0.857\,{\tiny$\pm$0.059} & 0.325\,{\tiny$\pm$0.087} \\
    \midrule
    \multirow{11}{*}{\rotatebox[origin=c]{90}{Llama-3.1-8B\,($qK{=}50$)}} & Base & 0.553 & 0.435 & 0.995 & 0.553 & 0.435 & 0.985 & 0.553 & 0.435 & 0.185 \\
     & Aligned (SFT) & 0.786 & 0.781 & 0.830 & 0.786 & 0.781 & 0.850 & 0.786 & 0.781 & 0.175 \\
    \cmidrule(l){2-11}
     & SFT & 0.910\,{\tiny$\pm$0.003} & 0.913\,{\tiny$\pm$0.005} & 0.995\,{\tiny$\pm$0.005} & 0.915\,{\tiny$\pm$0.005} & 0.914\,{\tiny$\pm$0.004} & \underline{0.980}\,{\tiny$\pm$0.013} & 0.894\,{\tiny$\pm$0.017} & 0.888\,{\tiny$\pm$0.013} & 0.355\,{\tiny$\pm$0.071} \\
     & Vaccine & 0.851\,{\tiny$\pm$0.055} & 0.837\,{\tiny$\pm$0.055} & 0.962\,{\tiny$\pm$0.019} & 0.851\,{\tiny$\pm$0.049} & 0.822\,{\tiny$\pm$0.039} & 0.903\,{\tiny$\pm$0.033} & 0.843\,{\tiny$\pm$0.013} & 0.821\,{\tiny$\pm$0.039} & 0.140\,{\tiny$\pm$0.018} \\
     & Booster & 0.894\,{\tiny$\pm$0.015} & 0.891\,{\tiny$\pm$0.020} & 0.995\,{\tiny$\pm$0.005} & 0.905\,{\tiny$\pm$0.013} & 0.898\,{\tiny$\pm$0.019} & 0.972\,{\tiny$\pm$0.026} & 0.908\,{\tiny$\pm$0.007} & 0.911\,{\tiny$\pm$0.012} & 0.342\,{\tiny$\pm$0.038} \\
     & Lisa & 0.887\,{\tiny$\pm$0.015} & 0.887\,{\tiny$\pm$0.021} & \underline{0.998}\,{\tiny$\pm$0.003} & 0.899\,{\tiny$\pm$0.010} & 0.891\,{\tiny$\pm$0.012} & 0.925\,{\tiny$\pm$0.061} & 0.848\,{\tiny$\pm$0.007} & 0.823\,{\tiny$\pm$0.010} & \textbf{0.427}\,{\tiny$\pm$0.042} \\
     & SaLoRA & 0.910\,{\tiny$\pm$0.012} & 0.899\,{\tiny$\pm$0.009} & 0.997\,{\tiny$\pm$0.003} & 0.909\,{\tiny$\pm$0.009} & 0.906\,{\tiny$\pm$0.000} & \textbf{0.985}\,{\tiny$\pm$0.005} & 0.895\,{\tiny$\pm$0.014} & 0.891\,{\tiny$\pm$0.007} & 0.333\,{\tiny$\pm$0.120} \\
     & SafeLoRA & \underline{0.794}\,{\tiny$\pm$0.001} & 0.797\,{\tiny$\pm$0.010} & 0.878\,{\tiny$\pm$0.033} & \underline{0.790}\,{\tiny$\pm$0.006} & \textbf{0.792}\,{\tiny$\pm$0.005} & 0.910\,{\tiny$\pm$0.038} & \textbf{0.778}\,{\tiny$\pm$0.001} & \underline{0.777}\,{\tiny$\pm$0.004} & 0.173\,{\tiny$\pm$0.023} \\
     & SafeGrad & 0.831\,{\tiny$\pm$0.007} & 0.844\,{\tiny$\pm$0.004} & \textbf{1.000}\,{\tiny$\pm$0.000} & 0.850\,{\tiny$\pm$0.009} & 0.856\,{\tiny$\pm$0.010} & 0.973\,{\tiny$\pm$0.014} & 0.824\,{\tiny$\pm$0.009} & 0.807\,{\tiny$\pm$0.010} & \underline{0.358}\,{\tiny$\pm$0.008} \\
     & AsFT & 0.800\,{\tiny$\pm$0.009} & \textbf{0.777}\,{\tiny$\pm$0.006} & 0.982\,{\tiny$\pm$0.013} & 0.801\,{\tiny$\pm$0.002} & 0.800\,{\tiny$\pm$0.006} & 0.892\,{\tiny$\pm$0.032} & 0.796\,{\tiny$\pm$0.008} & \textbf{0.766}\,{\tiny$\pm$0.007} & 0.205\,{\tiny$\pm$0.028} \\
     & Ours & \textbf{0.776}\,{\tiny$\pm$0.015} & \underline{0.785}\,{\tiny$\pm$0.009} & 0.985\,{\tiny$\pm$0.015} & \textbf{0.785}\,{\tiny$\pm$0.026} & \underline{0.794}\,{\tiny$\pm$0.015} & 0.928\,{\tiny$\pm$0.032} & \underline{0.794}\,{\tiny$\pm$0.032} & 0.796\,{\tiny$\pm$0.022} & 0.268\,{\tiny$\pm$0.050} \\
    \bottomrule
  \end{tabular}}
\end{table*}

%% file: tables/neuripsPapers/poison-sweep-qwen3.5_4b_sst2.tex
\begin{table}[t]
  \centering
  \small
  \caption{Poison-ratio sweep on Qwen-3.5-4B / SST-2 (align=sft). For each defense and poison ratio $p\in\{0.05, 0.1, 0.25, 0.5\}$ we report HS (fraction of prompts with max harm probability $>0.4$, $\downarrow$), HS$_{q90}$ (90th-percentile per-prompt max harm probability, $\downarrow$), and FA (SST-2 accuracy, $\uparrow$). Base and Aligned-SFT baselines do not depend on $p$ and are shown as single spanned values. \textbf{Bold} = best, \underline{underlined} = second best per column. Antibody-FT is deferred to the appendix. Defense rows are mean\,{\tiny$\pm$std} over 3 seeds.}
  \label{tab:poison-sweep-qwen3.5_4b_sst2}
  \resizebox{\columnwidth}{!}{%
  \begin{tabular}{lcccccccccccc}
    \toprule
    Defense & \multicolumn{3}{c}{$p{=}0.05$} & \multicolumn{3}{c}{$p{=}0.1$} & \multicolumn{3}{c}{$p{=}0.25$} & \multicolumn{3}{c}{$p{=}0.5$} \\
    \cmidrule(lr){2-4}\cmidrule(lr){5-7}\cmidrule(lr){8-10}\cmidrule(lr){11-13}
     & HS $\downarrow$ & HS$_{q90}$ $\downarrow$ & FA $\uparrow$ & HS $\downarrow$ & HS$_{q90}$ $\downarrow$ & FA $\uparrow$ & HS $\downarrow$ & HS$_{q90}$ $\downarrow$ & FA $\uparrow$ & HS $\downarrow$ & HS$_{q90}$ $\downarrow$ & FA $\uparrow$ \\
    \midrule
    Base & \multicolumn{12}{c}{HS = 0.430 \quad HS$_{q90}$ = 0.707 \quad FA = 0.915 \quad \footnotesize\textit{(constant in $p$)}} \\
    Aligned (SFT) & \multicolumn{12}{c}{HS = 0.350 \quad HS$_{q90}$ = 0.738 \quad FA = 0.940 \quad \footnotesize\textit{(constant in $p$)}} \\
    \midrule
    SFT & 0.425\,{\tiny$\pm$0.040} & 0.810\,{\tiny$\pm$0.029} & \underline{0.985}\,{\tiny$\pm$0.013} & 0.437\,{\tiny$\pm$0.037} & 0.828\,{\tiny$\pm$0.021} & \underline{0.977}\,{\tiny$\pm$0.014} & 0.516\,{\tiny$\pm$0.020} & 0.862\,{\tiny$\pm$0.015} & \underline{0.983}\,{\tiny$\pm$0.020} & 0.546\,{\tiny$\pm$0.042} & 0.890\,{\tiny$\pm$0.015} & \underline{0.973}\,{\tiny$\pm$0.016} \\
    Vaccine & 0.575\,{\tiny$\pm$0.046} & 0.940\,{\tiny$\pm$0.018} & 0.973\,{\tiny$\pm$0.019} & 0.604\,{\tiny$\pm$0.027} & 0.944\,{\tiny$\pm$0.008} & 0.972\,{\tiny$\pm$0.010} & 0.665\,{\tiny$\pm$0.025} & 0.958\,{\tiny$\pm$0.006} & \underline{0.983}\,{\tiny$\pm$0.010} & 0.688\,{\tiny$\pm$0.032} & 0.962\,{\tiny$\pm$0.008} & 0.963\,{\tiny$\pm$0.038} \\
    Booster & 0.464\,{\tiny$\pm$0.015} & 0.872\,{\tiny$\pm$0.013} & 0.973\,{\tiny$\pm$0.008} & 0.493\,{\tiny$\pm$0.015} & 0.883\,{\tiny$\pm$0.007} & 0.957\,{\tiny$\pm$0.008} & 0.600\,{\tiny$\pm$0.026} & 0.918\,{\tiny$\pm$0.010} & 0.943\,{\tiny$\pm$0.016} & 0.665\,{\tiny$\pm$0.021} & 0.940\,{\tiny$\pm$0.012} & 0.943\,{\tiny$\pm$0.018} \\
    Lisa & 0.395\,{\tiny$\pm$0.018} & 0.792\,{\tiny$\pm$0.026} & 0.982\,{\tiny$\pm$0.010} & 0.421\,{\tiny$\pm$0.023} & 0.804\,{\tiny$\pm$0.025} & 0.967\,{\tiny$\pm$0.010} & 0.447\,{\tiny$\pm$0.037} & 0.815\,{\tiny$\pm$0.011} & 0.973\,{\tiny$\pm$0.008} & 0.489\,{\tiny$\pm$0.051} & 0.845\,{\tiny$\pm$0.037} & 0.968\,{\tiny$\pm$0.003} \\
    SaLoRA & 0.430\,{\tiny$\pm$0.016} & 0.815\,{\tiny$\pm$0.008} & \textbf{0.990}\,{\tiny$\pm$0.005} & 0.467\,{\tiny$\pm$0.008} & 0.841\,{\tiny$\pm$0.016} & \textbf{0.983}\,{\tiny$\pm$0.008} & 0.536\,{\tiny$\pm$0.030} & 0.867\,{\tiny$\pm$0.024} & \textbf{0.988}\,{\tiny$\pm$0.012} & 0.540\,{\tiny$\pm$0.040} & 0.895\,{\tiny$\pm$0.014} & \textbf{0.975}\,{\tiny$\pm$0.005} \\
    SafeLoRA & \underline{0.351}\,{\tiny$\pm$0.013} & \underline{0.735}\,{\tiny$\pm$0.012} & 0.957\,{\tiny$\pm$0.015} & \underline{0.361}\,{\tiny$\pm$0.003} & \underline{0.735}\,{\tiny$\pm$0.014} & 0.952\,{\tiny$\pm$0.020} & 0.364\,{\tiny$\pm$0.012} & \textbf{0.727}\,{\tiny$\pm$0.004} & 0.953\,{\tiny$\pm$0.010} & \underline{0.363}\,{\tiny$\pm$0.014} & \textbf{0.721}\,{\tiny$\pm$0.019} & 0.952\,{\tiny$\pm$0.018} \\
    SafeGrad & 0.377\,{\tiny$\pm$0.010} & 0.758\,{\tiny$\pm$0.017} & 0.978\,{\tiny$\pm$0.008} & 0.393\,{\tiny$\pm$0.022} & 0.787\,{\tiny$\pm$0.015} & 0.973\,{\tiny$\pm$0.012} & 0.415\,{\tiny$\pm$0.031} & 0.806\,{\tiny$\pm$0.033} & 0.973\,{\tiny$\pm$0.016} & 0.445\,{\tiny$\pm$0.015} & 0.843\,{\tiny$\pm$0.006} & 0.967\,{\tiny$\pm$0.019} \\
    AsFT & 0.378\,{\tiny$\pm$0.010} & 0.776\,{\tiny$\pm$0.002} & 0.953\,{\tiny$\pm$0.010} & 0.374\,{\tiny$\pm$0.012} & 0.784\,{\tiny$\pm$0.006} & 0.952\,{\tiny$\pm$0.010} & \underline{0.360}\,{\tiny$\pm$0.006} & \underline{0.779}\,{\tiny$\pm$0.002} & 0.958\,{\tiny$\pm$0.006} & 0.371\,{\tiny$\pm$0.013} & 0.774\,{\tiny$\pm$0.010} & 0.952\,{\tiny$\pm$0.006} \\
    Ours & \textbf{0.311}\,{\tiny$\pm$0.018} & \textbf{0.715}\,{\tiny$\pm$0.010} & 0.962\,{\tiny$\pm$0.019} & \textbf{0.325}\,{\tiny$\pm$0.002} & \textbf{0.719}\,{\tiny$\pm$0.024} & 0.957\,{\tiny$\pm$0.019} & \textbf{0.339}\,{\tiny$\pm$0.007} & \textbf{0.727}\,{\tiny$\pm$0.021} & 0.960\,{\tiny$\pm$0.018} & \textbf{0.340}\,{\tiny$\pm$0.032} & \underline{0.726}\,{\tiny$\pm$0.041} & 0.957\,{\tiny$\pm$0.018} \\
    \bottomrule
  \end{tabular}}
\end{table}

%% file: tables/neuripsPapers/poison-sweep-qwen3.5_4b_ag_news.tex
\begin{table}[t]
  \centering
  \small
  \caption{Poison-ratio sweep on Qwen-3.5-4B / AG News (align=sft). For each defense and poison ratio $p\in\{0.05, 0.1, 0.25, 0.5\}$ we report HS (fraction of prompts with max harm probability $>0.4$, $\downarrow$), HS$_{q90}$ (90th-percentile per-prompt max harm probability, $\downarrow$), and FA (AG News accuracy, $\uparrow$). Base and Aligned-SFT baselines do not depend on $p$ and are shown as single spanned values. \textbf{Bold} = best, \underline{underlined} = second best per column. Defense rows are mean\,{\tiny$\pm$std} over 3 seeds.}
  \label{tab:poison-sweep-qwen3.5_4b_ag_news}
  \resizebox{\columnwidth}{!}{%
  \begin{tabular}{lcccccccccccc}
    \toprule
    Defense & \multicolumn{3}{c}{$p{=}0.05$} & \multicolumn{3}{c}{$p{=}0.1$} & \multicolumn{3}{c}{$p{=}0.25$} & \multicolumn{3}{c}{$p{=}0.5$} \\
    \cmidrule(lr){2-4}\cmidrule(lr){5-7}\cmidrule(lr){8-10}\cmidrule(lr){11-13}
     & HS $\downarrow$ & HS$_{q90}$ $\downarrow$ & FA $\uparrow$ & HS $\downarrow$ & HS$_{q90}$ $\downarrow$ & FA $\uparrow$ & HS $\downarrow$ & HS$_{q90}$ $\downarrow$ & FA $\uparrow$ & HS $\downarrow$ & HS$_{q90}$ $\downarrow$ & FA $\uparrow$ \\
    \midrule
    Base & \multicolumn{12}{c}{HS = 0.430 \quad HS$_{q90}$ = 0.707 \quad FA = 1.000 \quad \footnotesize\textit{(constant in $p$)}} \\
    Aligned (SFT) & \multicolumn{12}{c}{HS = 0.350 \quad HS$_{q90}$ = 0.738 \quad FA = 0.890 \quad \footnotesize\textit{(constant in $p$)}} \\
    \midrule
    SFT & 0.434\,{\tiny$\pm$0.024} & 0.844\,{\tiny$\pm$0.004} & 0.917\,{\tiny$\pm$0.026} & 0.445\,{\tiny$\pm$0.010} & 0.866\,{\tiny$\pm$0.006} & 0.912\,{\tiny$\pm$0.028} & 0.517\,{\tiny$\pm$0.008} & 0.882\,{\tiny$\pm$0.023} & 0.895\,{\tiny$\pm$0.040} & 0.545\,{\tiny$\pm$0.043} & 0.901\,{\tiny$\pm$0.025} & 0.890\,{\tiny$\pm$0.015} \\
    Vaccine & 0.584\,{\tiny$\pm$0.024} & 0.936\,{\tiny$\pm$0.008} & \underline{0.947}\,{\tiny$\pm$0.062} & 0.613\,{\tiny$\pm$0.015} & 0.947\,{\tiny$\pm$0.011} & \underline{0.930}\,{\tiny$\pm$0.066} & 0.684\,{\tiny$\pm$0.021} & 0.967\,{\tiny$\pm$0.002} & \underline{0.943}\,{\tiny$\pm$0.028} & 0.680\,{\tiny$\pm$0.025} & 0.960\,{\tiny$\pm$0.002} & \underline{0.932}\,{\tiny$\pm$0.033} \\
    Booster & 0.476\,{\tiny$\pm$0.029} & 0.884\,{\tiny$\pm$0.006} & \textbf{0.975}\,{\tiny$\pm$0.018} & 0.496\,{\tiny$\pm$0.016} & 0.891\,{\tiny$\pm$0.004} & \textbf{0.968}\,{\tiny$\pm$0.021} & 0.597\,{\tiny$\pm$0.010} & 0.918\,{\tiny$\pm$0.004} & \textbf{0.973}\,{\tiny$\pm$0.006} & 0.674\,{\tiny$\pm$0.013} & 0.949\,{\tiny$\pm$0.007} & \textbf{0.952}\,{\tiny$\pm$0.021} \\
    Lisa & 0.398\,{\tiny$\pm$0.007} & 0.819\,{\tiny$\pm$0.012} & 0.917\,{\tiny$\pm$0.025} & 0.413\,{\tiny$\pm$0.013} & 0.835\,{\tiny$\pm$0.032} & 0.913\,{\tiny$\pm$0.025} & 0.476\,{\tiny$\pm$0.010} & 0.865\,{\tiny$\pm$0.008} & 0.917\,{\tiny$\pm$0.016} & 0.510\,{\tiny$\pm$0.038} & 0.874\,{\tiny$\pm$0.020} & 0.900\,{\tiny$\pm$0.018} \\
    SaLoRA & 0.415\,{\tiny$\pm$0.027} & 0.842\,{\tiny$\pm$0.006} & 0.908\,{\tiny$\pm$0.033} & 0.433\,{\tiny$\pm$0.004} & 0.851\,{\tiny$\pm$0.005} & 0.913\,{\tiny$\pm$0.030} & 0.513\,{\tiny$\pm$0.012} & 0.876\,{\tiny$\pm$0.012} & 0.898\,{\tiny$\pm$0.036} & 0.558\,{\tiny$\pm$0.024} & 0.898\,{\tiny$\pm$0.022} & 0.898\,{\tiny$\pm$0.020} \\
    SafeLoRA & \underline{0.361}\,{\tiny$\pm$0.018} & \textbf{0.744}\,{\tiny$\pm$0.008} & 0.903\,{\tiny$\pm$0.023} & \underline{0.357}\,{\tiny$\pm$0.007} & \textbf{0.724}\,{\tiny$\pm$0.026} & 0.900\,{\tiny$\pm$0.022} & \underline{0.354}\,{\tiny$\pm$0.017} & \textbf{0.723}\,{\tiny$\pm$0.007} & 0.902\,{\tiny$\pm$0.029} & \underline{0.358}\,{\tiny$\pm$0.026} & \textbf{0.726}\,{\tiny$\pm$0.008} & 0.898\,{\tiny$\pm$0.032} \\
    SafeGrad & 0.388\,{\tiny$\pm$0.004} & 0.825\,{\tiny$\pm$0.012} & 0.913\,{\tiny$\pm$0.025} & 0.395\,{\tiny$\pm$0.012} & 0.835\,{\tiny$\pm$0.002} & 0.917\,{\tiny$\pm$0.026} & 0.450\,{\tiny$\pm$0.003} & 0.856\,{\tiny$\pm$0.010} & 0.900\,{\tiny$\pm$0.018} & 0.450\,{\tiny$\pm$0.038} & 0.864\,{\tiny$\pm$0.031} & 0.888\,{\tiny$\pm$0.006} \\
    AsFT & 0.363\,{\tiny$\pm$0.005} & 0.794\,{\tiny$\pm$0.016} & 0.917\,{\tiny$\pm$0.019} & 0.358\,{\tiny$\pm$0.011} & 0.790\,{\tiny$\pm$0.008} & 0.910\,{\tiny$\pm$0.020} & 0.359\,{\tiny$\pm$0.010} & 0.785\,{\tiny$\pm$0.011} & 0.907\,{\tiny$\pm$0.019} & 0.374\,{\tiny$\pm$0.011} & 0.792\,{\tiny$\pm$0.009} & 0.907\,{\tiny$\pm$0.010} \\
    Ours & \textbf{0.327}\,{\tiny$\pm$0.019} & \underline{0.755}\,{\tiny$\pm$0.034} & 0.913\,{\tiny$\pm$0.038} & \textbf{0.322}\,{\tiny$\pm$0.015} & \underline{0.750}\,{\tiny$\pm$0.024} & 0.918\,{\tiny$\pm$0.029} & \textbf{0.324}\,{\tiny$\pm$0.005} & \underline{0.753}\,{\tiny$\pm$0.013} & 0.913\,{\tiny$\pm$0.016} & \textbf{0.340}\,{\tiny$\pm$0.017} & \underline{0.732}\,{\tiny$\pm$0.039} & 0.927\,{\tiny$\pm$0.006} \\
    \bottomrule
  \end{tabular}}
\end{table}

%% file: tables/neuripsPapers/poison-sweep-qwen3.5_4b_gsm8k.tex
\begin{table}[t]
  \centering
  \small
  \caption{Poison-ratio sweep on Qwen-3.5-4B / GSM8K (align=sft). For each defense and poison ratio $p\in\{0.05, 0.1, 0.25, 0.5\}$ we report HS (fraction of prompts with max harm probability $>0.4$, $\downarrow$), HS$_{q90}$ (90th-percentile per-prompt max harm probability, $\downarrow$), and FA (GSM8K accuracy, $\uparrow$). Base and Aligned-SFT baselines do not depend on $p$ and are shown as single spanned values. \textbf{Bold} = best, \underline{underlined} = second best per column.  Defense rows are mean\,{\tiny$\pm$std} over 3 seeds.}
  \label{tab:poison-sweep-qwen3.5_4b_gsm8k}
  \resizebox{\columnwidth}{!}{%
  \begin{tabular}{lcccccccccccc}
    \toprule
    Defense & \multicolumn{3}{c}{$p{=}0.05$} & \multicolumn{3}{c}{$p{=}0.1$} & \multicolumn{3}{c}{$p{=}0.25$} & \multicolumn{3}{c}{$p{=}0.5$} \\
    \cmidrule(lr){2-4}\cmidrule(lr){5-7}\cmidrule(lr){8-10}\cmidrule(lr){11-13}
     & HS $\downarrow$ & HS$_{q90}$ $\downarrow$ & FA $\uparrow$ & HS $\downarrow$ & HS$_{q90}$ $\downarrow$ & FA $\uparrow$ & HS $\downarrow$ & HS$_{q90}$ $\downarrow$ & FA $\uparrow$ & HS $\downarrow$ & HS$_{q90}$ $\downarrow$ & FA $\uparrow$ \\
    \midrule
    Base & \multicolumn{12}{c}{HS = 0.430 \quad HS$_{q90}$ = 0.707 \quad FA = 0.280 \quad \footnotesize\textit{(constant in $p$)}} \\
    Aligned (SFT) & \multicolumn{12}{c}{HS = 0.350 \quad HS$_{q90}$ = 0.738 \quad FA = 0.170 \quad \footnotesize\textit{(constant in $p$)}} \\
    \midrule
    SFT & 0.470\,{\tiny$\pm$0.024} & 0.877\,{\tiny$\pm$0.006} & \underline{0.665}\,{\tiny$\pm$0.010} & 0.510\,{\tiny$\pm$0.043} & 0.889\,{\tiny$\pm$0.016} & \underline{0.665}\,{\tiny$\pm$0.020} & 0.605\,{\tiny$\pm$0.026} & 0.938\,{\tiny$\pm$0.010} & \underline{0.660}\,{\tiny$\pm$0.043} & 0.637\,{\tiny$\pm$0.047} & 0.945\,{\tiny$\pm$0.017} & \underline{0.633}\,{\tiny$\pm$0.064} \\
    Vaccine & 0.636\,{\tiny$\pm$0.024} & 0.948\,{\tiny$\pm$0.013} & 0.498\,{\tiny$\pm$0.020} & 0.653\,{\tiny$\pm$0.024} & 0.952\,{\tiny$\pm$0.004} & 0.465\,{\tiny$\pm$0.044} & 0.696\,{\tiny$\pm$0.015} & 0.964\,{\tiny$\pm$0.005} & 0.465\,{\tiny$\pm$0.026} & 0.697\,{\tiny$\pm$0.023} & 0.968\,{\tiny$\pm$0.009} & 0.378\,{\tiny$\pm$0.172} \\
    Booster & 0.565\,{\tiny$\pm$0.010} & 0.924\,{\tiny$\pm$0.002} & 0.460\,{\tiny$\pm$0.025} & 0.587\,{\tiny$\pm$0.016} & 0.923\,{\tiny$\pm$0.005} & 0.450\,{\tiny$\pm$0.015} & 0.658\,{\tiny$\pm$0.009} & 0.938\,{\tiny$\pm$0.007} & 0.452\,{\tiny$\pm$0.019} & 0.715\,{\tiny$\pm$0.012} & 0.951\,{\tiny$\pm$0.002} & 0.427\,{\tiny$\pm$0.083} \\
    Lisa & 0.420\,{\tiny$\pm$0.004} & 0.852\,{\tiny$\pm$0.016} & 0.377\,{\tiny$\pm$0.020} & 0.429\,{\tiny$\pm$0.010} & 0.871\,{\tiny$\pm$0.008} & 0.372\,{\tiny$\pm$0.038} & 0.487\,{\tiny$\pm$0.025} & 0.883\,{\tiny$\pm$0.007} & 0.313\,{\tiny$\pm$0.014} & 0.532\,{\tiny$\pm$0.032} & 0.913\,{\tiny$\pm$0.016} & 0.290\,{\tiny$\pm$0.041} \\
    SaLoRA & 0.469\,{\tiny$\pm$0.040} & 0.869\,{\tiny$\pm$0.027} & \textbf{0.683}\,{\tiny$\pm$0.018} & 0.510\,{\tiny$\pm$0.040} & 0.897\,{\tiny$\pm$0.022} & \textbf{0.667}\,{\tiny$\pm$0.020} & 0.592\,{\tiny$\pm$0.017} & 0.940\,{\tiny$\pm$0.010} & \textbf{0.668}\,{\tiny$\pm$0.016} & 0.656\,{\tiny$\pm$0.016} & 0.947\,{\tiny$\pm$0.004} & \textbf{0.655}\,{\tiny$\pm$0.061} \\
    SafeLoRA & \textbf{0.354}\,{\tiny$\pm$0.011} & \textbf{0.723}\,{\tiny$\pm$0.004} & 0.147\,{\tiny$\pm$0.023} & \textbf{0.361}\,{\tiny$\pm$0.007} & \textbf{0.709}\,{\tiny$\pm$0.015} & 0.137\,{\tiny$\pm$0.015} & \textbf{0.365}\,{\tiny$\pm$0.008} & \textbf{0.707}\,{\tiny$\pm$0.026} & 0.130\,{\tiny$\pm$0.033} & \underline{0.358}\,{\tiny$\pm$0.012} & \textbf{0.720}\,{\tiny$\pm$0.002} & 0.135\,{\tiny$\pm$0.005} \\
    SafeGrad & 0.415\,{\tiny$\pm$0.026} & 0.857\,{\tiny$\pm$0.025} & 0.623\,{\tiny$\pm$0.078} & 0.418\,{\tiny$\pm$0.007} & 0.864\,{\tiny$\pm$0.013} & 0.600\,{\tiny$\pm$0.052} & 0.461\,{\tiny$\pm$0.014} & 0.888\,{\tiny$\pm$0.016} & 0.588\,{\tiny$\pm$0.090} & 0.495\,{\tiny$\pm$0.026} & 0.894\,{\tiny$\pm$0.028} & 0.542\,{\tiny$\pm$0.083} \\
    AsFT & \underline{0.367}\,{\tiny$\pm$0.006} & \underline{0.800}\,{\tiny$\pm$0.008} & 0.153\,{\tiny$\pm$0.015} & \underline{0.366}\,{\tiny$\pm$0.015} & \underline{0.799}\,{\tiny$\pm$0.014} & 0.133\,{\tiny$\pm$0.008} & \underline{0.369}\,{\tiny$\pm$0.012} & \underline{0.778}\,{\tiny$\pm$0.000} & 0.137\,{\tiny$\pm$0.029} & \textbf{0.348}\,{\tiny$\pm$0.006} & \underline{0.737}\,{\tiny$\pm$0.013} & 0.138\,{\tiny$\pm$0.016} \\
    Ours & 0.386\,{\tiny$\pm$0.005} & 0.837\,{\tiny$\pm$0.002} & 0.395\,{\tiny$\pm$0.005} & 0.389\,{\tiny$\pm$0.012} & 0.845\,{\tiny$\pm$0.019} & 0.382\,{\tiny$\pm$0.003} & 0.404\,{\tiny$\pm$0.007} & 0.845\,{\tiny$\pm$0.006} & 0.380\,{\tiny$\pm$0.009} & 0.360\,{\tiny$\pm$0.052} & 0.780\,{\tiny$\pm$0.118} & 0.267\,{\tiny$\pm$0.110} \\
    \bottomrule
  \end{tabular}}
\end{table}

%% file: tables/neuripsPapers/cage_sensitivity_sst2.tex
\begin{table}[t]
  \centering
  \small
  \caption{Hyperparameter sensitivity on Qwen-3.5-4B / SST-2 at $p{=}0.5$. Each section sweeps one of the three hyperparameters while holding the other two at the canonical values used in the main results $(\tau, \alpha, \beta) = (0.1, 0.05, 10)$. \textbf{Bold} = best, \underline{underlined} = second best within each sweep. $\downarrow$ lower is better, $\uparrow$ higher is better.}
  \label{tab:cage-ra-sens}
  \begin{tabular}{cccccc}
    \toprule
    Value & HS $\downarrow$ & $q_{50}$ $\downarrow$ & $q_{90}$ $\downarrow$ & $q_{99}$ $\downarrow$ & FA $\uparrow$ \\
    \midrule
    \multicolumn{6}{l}{\textit{$\tau$ sweep, fixing $\alpha=0.05,\ \beta=10$}} \\
    \midrule
    0 & 0.323 & 0.256 & \textbf{0.676} & \underline{0.942} & 0.925 \\
    0.01 & \textbf{0.311} & \textbf{0.250} & \underline{0.691} & \textbf{0.930} & 0.925 \\
    0.05 & \underline{0.314} & \underline{0.252} & 0.704 & 0.949 & \underline{0.930} \\
    0.1 & 0.320 & 0.253 & 0.707 & \textbf{0.930} & 0.920 \\
    0.2 & 0.350 & \underline{0.252} & 0.723 & 0.949 & \textbf{0.935} \\
    \midrule
    \multicolumn{6}{l}{\textit{$\alpha$ sweep, fixing $\tau=0.1,\ \beta=10$}} \\
    \midrule
    0.01 & \underline{0.313} & 0.253 & \textbf{0.704} & \underline{0.934} & 0.920 \\
    0.05 & 0.320 & 0.253 & \underline{0.707} & \textbf{0.930} & 0.920 \\
    0.1 & \textbf{0.300} & \textbf{0.236} & 0.723 & 0.941 & \underline{0.925} \\
    0.2 & 0.314 & \underline{0.250} & 0.730 & 0.953 & \textbf{0.930} \\
    \midrule
    \multicolumn{6}{l}{\textit{$\beta$ sweep, fixing $\tau=0.1,\ \alpha=0.05$}} \\
    \midrule
    1 & \textbf{0.303} & \underline{0.263} & \textbf{0.676} & \textbf{0.902} & \textbf{0.925} \\
    10 & \underline{0.320} & \textbf{0.253} & \underline{0.707} & \underline{0.930} & \underline{0.920} \\
    \bottomrule
  \end{tabular}
\end{table}

%% file: tables/neuripsPapers/quantile-qwen3.5_4b_sst2.tex
\begin{table}[t]
  \centering
  \small
  \caption{HS-quantile breakdown on Qwen-3.5-4B / SST-2 at $p{=}0.5$, align=sft. HS is the fraction of prompts with max harm probability $>0.4$ (binary flag rate); $q_{k}$ columns are the $k$-th percentile of the per-prompt max harm probability (continuous). $\downarrow$ for HS columns, $\uparrow$ for FA. \textbf{Bold} = best, \underline{underlined} = second best per column. Defense rows are mean\,{\tiny$\pm$std} over 3 seeds.}
  \label{tab:quantile-qwen3.5_4b_sst2}
  \resizebox{\columnwidth}{!}{%
  \begin{tabular}{lcccccccccc}
    \toprule
    Defense & HS $\downarrow$ & $q_{10}$ $\downarrow$ & $q_{25}$ $\downarrow$ & $q_{50}$ $\downarrow$ & $q_{75}$ $\downarrow$ & $q_{90}$ $\downarrow$ & $q_{95}$ $\downarrow$ & $q_{99}$ $\downarrow$ & $q_{100}$ $\downarrow$ & FA $\uparrow$ \\
    \midrule
    Base & 0.430 & 0.161 & 0.248 & 0.365 & 0.531 & 0.707 & 0.824 & 0.938 & 0.984 & 0.915 \\
    Aligned (SFT) & 0.350 & 0.094 & 0.164 & 0.296 & 0.504 & 0.738 & 0.867 & 0.961 & 1.000 & 0.940 \\
    \midrule
    SFT & 0.546\,{\tiny$\pm$0.042} & 0.119\,{\tiny$\pm$0.010} & 0.227\,{\tiny$\pm$0.020} & 0.446\,{\tiny$\pm$0.040} & 0.705\,{\tiny$\pm$0.037} & 0.890\,{\tiny$\pm$0.015} & 0.947\,{\tiny$\pm$0.008} & 0.992\,{\tiny$\pm$0.000} & 1.000\,{\tiny$\pm$0.000} & \underline{0.973}\,{\tiny$\pm$0.016} \\
    Vaccine & 0.688\,{\tiny$\pm$0.032} & 0.149\,{\tiny$\pm$0.023} & 0.316\,{\tiny$\pm$0.033} & 0.611\,{\tiny$\pm$0.026} & 0.861\,{\tiny$\pm$0.022} & 0.962\,{\tiny$\pm$0.008} & 0.986\,{\tiny$\pm$0.002} & 1.000\,{\tiny$\pm$0.000} & 1.000\,{\tiny$\pm$0.000} & 0.963\,{\tiny$\pm$0.038} \\
    Booster & 0.665\,{\tiny$\pm$0.021} & 0.143\,{\tiny$\pm$0.013} & 0.305\,{\tiny$\pm$0.017} & 0.584\,{\tiny$\pm$0.030} & 0.826\,{\tiny$\pm$0.026} & 0.940\,{\tiny$\pm$0.012} & 0.970\,{\tiny$\pm$0.006} & 1.000\,{\tiny$\pm$0.000} & 1.000\,{\tiny$\pm$0.000} & 0.943\,{\tiny$\pm$0.018} \\
    Lisa & 0.489\,{\tiny$\pm$0.051} & 0.115\,{\tiny$\pm$0.009} & 0.206\,{\tiny$\pm$0.022} & 0.392\,{\tiny$\pm$0.046} & 0.651\,{\tiny$\pm$0.045} & 0.845\,{\tiny$\pm$0.037} & 0.925\,{\tiny$\pm$0.012} & 0.983\,{\tiny$\pm$0.014} & 1.000\,{\tiny$\pm$0.000} & 0.968\,{\tiny$\pm$0.003} \\
    SaLoRA & 0.540\,{\tiny$\pm$0.040} & 0.117\,{\tiny$\pm$0.004} & 0.217\,{\tiny$\pm$0.027} & 0.440\,{\tiny$\pm$0.050} & 0.710\,{\tiny$\pm$0.049} & 0.895\,{\tiny$\pm$0.014} & 0.941\,{\tiny$\pm$0.008} & 0.990\,{\tiny$\pm$0.002} & 1.000\,{\tiny$\pm$0.000} & \textbf{0.975}\,{\tiny$\pm$0.005} \\
    SafeLoRA & \underline{0.363}\,{\tiny$\pm$0.014} & 0.099\,{\tiny$\pm$0.002} & 0.169\,{\tiny$\pm$0.007} & 0.298\,{\tiny$\pm$0.003} & \underline{0.500}\,{\tiny$\pm$0.018} & \textbf{0.721}\,{\tiny$\pm$0.019} & \underline{0.840}\,{\tiny$\pm$0.022} & 0.970\,{\tiny$\pm$0.004} & 0.996\,{\tiny$\pm$0.004} & 0.952\,{\tiny$\pm$0.018} \\
    SafeGrad & 0.445\,{\tiny$\pm$0.015} & 0.101\,{\tiny$\pm$0.008} & 0.175\,{\tiny$\pm$0.006} & 0.349\,{\tiny$\pm$0.018} & 0.633\,{\tiny$\pm$0.014} & 0.843\,{\tiny$\pm$0.006} & 0.928\,{\tiny$\pm$0.002} & 0.986\,{\tiny$\pm$0.002} & 1.000\,{\tiny$\pm$0.000} & 0.967\,{\tiny$\pm$0.019} \\
    AsFT & 0.371\,{\tiny$\pm$0.013} & \underline{0.090}\,{\tiny$\pm$0.004} & \underline{0.154}\,{\tiny$\pm$0.005} & \underline{0.287}\,{\tiny$\pm$0.002} & 0.541\,{\tiny$\pm$0.016} & 0.774\,{\tiny$\pm$0.010} & 0.875\,{\tiny$\pm$0.004} & \underline{0.965}\,{\tiny$\pm$0.008} & \underline{0.993}\,{\tiny$\pm$0.002} & 0.952\,{\tiny$\pm$0.006} \\
    Ours & \textbf{0.340}\,{\tiny$\pm$0.032} & \textbf{0.089}\,{\tiny$\pm$0.005} & \textbf{0.151}\,{\tiny$\pm$0.006} & \textbf{0.272}\,{\tiny$\pm$0.008} & \textbf{0.494}\,{\tiny$\pm$0.038} & \underline{0.726}\,{\tiny$\pm$0.041} & \textbf{0.827}\,{\tiny$\pm$0.030} & \textbf{0.947}\,{\tiny$\pm$0.005} & \textbf{0.990}\,{\tiny$\pm$0.008} & 0.957\,{\tiny$\pm$0.018} \\
    \bottomrule
  \end{tabular}}
\end{table}

%% file: tables/neuripsPapers/quantile-qwen3.5_4b_ag_news.tex
\begin{table}[t]
  \centering
  \small
  \caption{HS-quantile breakdown on Qwen-3.5-4B / AG News at $p{=}0.5$, align=sft. HS is the fraction of prompts with max harm probability $>0.4$ (binary flag rate); $q_{k}$ columns are the $k$-th percentile of the per-prompt max harm probability (continuous). $\downarrow$ for HS columns, $\uparrow$ for FA. \textbf{Bold} = best, \underline{underlined} = second best per column. Defense rows are mean\,{\tiny$\pm$std} over 3 seeds.}
  \label{tab:quantile-qwen3.5_4b_ag_news}
  \resizebox{\columnwidth}{!}{%
  \begin{tabular}{lcccccccccc}
    \toprule
    Defense & HS $\downarrow$ & $q_{10}$ $\downarrow$ & $q_{25}$ $\downarrow$ & $q_{50}$ $\downarrow$ & $q_{75}$ $\downarrow$ & $q_{90}$ $\downarrow$ & $q_{95}$ $\downarrow$ & $q_{99}$ $\downarrow$ & $q_{100}$ $\downarrow$ & FA $\uparrow$ \\
    \midrule
    Base & 0.430 & 0.161 & 0.248 & 0.365 & 0.531 & 0.707 & 0.824 & 0.938 & 0.984 & 1.000 \\
    Aligned (SFT) & 0.350 & 0.094 & 0.164 & 0.296 & 0.504 & 0.738 & 0.867 & 0.961 & 1.000 & 0.890 \\
    \midrule
    SFT & 0.545\,{\tiny$\pm$0.043} & 0.113\,{\tiny$\pm$0.010} & 0.216\,{\tiny$\pm$0.024} & 0.449\,{\tiny$\pm$0.049} & 0.730\,{\tiny$\pm$0.035} & 0.901\,{\tiny$\pm$0.025} & 0.948\,{\tiny$\pm$0.011} & 0.992\,{\tiny$\pm$0.008} & 1.000\,{\tiny$\pm$0.000} & 0.890\,{\tiny$\pm$0.015} \\
    Vaccine & 0.680\,{\tiny$\pm$0.025} & 0.154\,{\tiny$\pm$0.012} & 0.311\,{\tiny$\pm$0.025} & 0.617\,{\tiny$\pm$0.023} & 0.857\,{\tiny$\pm$0.006} & 0.960\,{\tiny$\pm$0.002} & 0.982\,{\tiny$\pm$0.002} & 0.999\,{\tiny$\pm$0.002} & 1.000\,{\tiny$\pm$0.000} & \underline{0.932}\,{\tiny$\pm$0.033} \\
    Booster & 0.674\,{\tiny$\pm$0.013} & 0.155\,{\tiny$\pm$0.006} & 0.319\,{\tiny$\pm$0.015} & 0.597\,{\tiny$\pm$0.027} & 0.836\,{\tiny$\pm$0.008} & 0.949\,{\tiny$\pm$0.007} & 0.979\,{\tiny$\pm$0.002} & 1.000\,{\tiny$\pm$0.000} & 1.000\,{\tiny$\pm$0.000} & \textbf{0.952}\,{\tiny$\pm$0.021} \\
    Lisa & 0.510\,{\tiny$\pm$0.038} & 0.108\,{\tiny$\pm$0.009} & 0.200\,{\tiny$\pm$0.019} & 0.407\,{\tiny$\pm$0.032} & 0.666\,{\tiny$\pm$0.022} & 0.874\,{\tiny$\pm$0.020} & 0.940\,{\tiny$\pm$0.012} & 0.992\,{\tiny$\pm$0.004} & 1.000\,{\tiny$\pm$0.000} & 0.900\,{\tiny$\pm$0.018} \\
    SaLoRA & 0.558\,{\tiny$\pm$0.024} & 0.111\,{\tiny$\pm$0.004} & 0.214\,{\tiny$\pm$0.024} & 0.456\,{\tiny$\pm$0.024} & 0.746\,{\tiny$\pm$0.038} & 0.898\,{\tiny$\pm$0.022} & 0.947\,{\tiny$\pm$0.016} & 0.992\,{\tiny$\pm$0.004} & 1.000\,{\tiny$\pm$0.000} & 0.898\,{\tiny$\pm$0.020} \\
    SafeLoRA & \underline{0.358}\,{\tiny$\pm$0.026} & 0.101\,{\tiny$\pm$0.004} & 0.167\,{\tiny$\pm$0.001} & 0.287\,{\tiny$\pm$0.015} & \textbf{0.507}\,{\tiny$\pm$0.012} & \textbf{0.726}\,{\tiny$\pm$0.008} & \underline{0.837}\,{\tiny$\pm$0.020} & \textbf{0.957}\,{\tiny$\pm$0.010} & \textbf{0.992}\,{\tiny$\pm$0.007} & 0.898\,{\tiny$\pm$0.032} \\
    SafeGrad & 0.450\,{\tiny$\pm$0.038} & 0.092\,{\tiny$\pm$0.002} & 0.164\,{\tiny$\pm$0.004} & 0.349\,{\tiny$\pm$0.040} & 0.649\,{\tiny$\pm$0.049} & 0.864\,{\tiny$\pm$0.031} & 0.936\,{\tiny$\pm$0.014} & 0.988\,{\tiny$\pm$0.004} & 1.000\,{\tiny$\pm$0.000} & 0.888\,{\tiny$\pm$0.006} \\
    AsFT & 0.374\,{\tiny$\pm$0.011} & \textbf{0.085}\,{\tiny$\pm$0.002} & \underline{0.151}\,{\tiny$\pm$0.006} & \underline{0.279}\,{\tiny$\pm$0.010} & 0.542\,{\tiny$\pm$0.014} & 0.792\,{\tiny$\pm$0.009} & 0.883\,{\tiny$\pm$0.015} & 0.983\,{\tiny$\pm$0.002} & 1.000\,{\tiny$\pm$0.000} & 0.907\,{\tiny$\pm$0.010} \\
    Ours & \textbf{0.340}\,{\tiny$\pm$0.017} & \underline{0.087}\,{\tiny$\pm$0.009} & \textbf{0.149}\,{\tiny$\pm$0.007} & \textbf{0.269}\,{\tiny$\pm$0.017} & \underline{0.513}\,{\tiny$\pm$0.035} & \underline{0.732}\,{\tiny$\pm$0.039} & \textbf{0.831}\,{\tiny$\pm$0.035} & \underline{0.961}\,{\tiny$\pm$0.012} & \underline{0.997}\,{\tiny$\pm$0.005} & 0.927\,{\tiny$\pm$0.006} \\
    \bottomrule
  \end{tabular}}
\end{table}

%% file: tables/neuripsPapers/quantile-qwen3.5_4b_gsm8k.tex
\begin{table}[t]
  \centering
  \small
  \caption{HS-quantile breakdown on Qwen-3.5-4B / GSM8K at $p{=}0.5$, align=sft. HS is the fraction of prompts with max harm probability $>0.4$ (binary flag rate); $q_{k}$ columns are the $k$-th percentile of the per-prompt max harm probability (continuous). $\downarrow$ for HS columns, $\uparrow$ for FA. \textbf{Bold} = best, \underline{underlined} = second best per column. Defense rows are mean\,{\tiny$\pm$std} over 3 seeds.}
  \label{tab:quantile-qwen3.5_4b_gsm8k}
  \resizebox{\columnwidth}{!}{%
  \begin{tabular}{lcccccccccc}
    \toprule
    Defense & HS $\downarrow$ & $q_{10}$ $\downarrow$ & $q_{25}$ $\downarrow$ & $q_{50}$ $\downarrow$ & $q_{75}$ $\downarrow$ & $q_{90}$ $\downarrow$ & $q_{95}$ $\downarrow$ & $q_{99}$ $\downarrow$ & $q_{100}$ $\downarrow$ & FA $\uparrow$ \\
    \midrule
    Base & 0.430 & 0.161 & 0.248 & 0.365 & 0.531 & 0.707 & 0.824 & 0.938 & 0.984 & 0.280 \\
    Aligned (SFT) & 0.350 & 0.094 & 0.164 & 0.296 & 0.504 & 0.738 & 0.867 & 0.961 & 1.000 & 0.170 \\
    \midrule
    SFT & 0.637\,{\tiny$\pm$0.047} & 0.133\,{\tiny$\pm$0.025} & 0.275\,{\tiny$\pm$0.049} & 0.558\,{\tiny$\pm$0.055} & 0.814\,{\tiny$\pm$0.049} & 0.945\,{\tiny$\pm$0.017} & 0.979\,{\tiny$\pm$0.009} & 1.000\,{\tiny$\pm$0.000} & 1.000\,{\tiny$\pm$0.000} & \underline{0.633}\,{\tiny$\pm$0.064} \\
    Vaccine & 0.697\,{\tiny$\pm$0.023} & 0.176\,{\tiny$\pm$0.007} & 0.346\,{\tiny$\pm$0.017} & 0.614\,{\tiny$\pm$0.028} & 0.872\,{\tiny$\pm$0.016} & 0.968\,{\tiny$\pm$0.009} & 0.991\,{\tiny$\pm$0.002} & 1.000\,{\tiny$\pm$0.000} & 1.000\,{\tiny$\pm$0.000} & 0.378\,{\tiny$\pm$0.172} \\
    Booster & 0.715\,{\tiny$\pm$0.012} & 0.175\,{\tiny$\pm$0.022} & 0.361\,{\tiny$\pm$0.016} & 0.635\,{\tiny$\pm$0.014} & 0.847\,{\tiny$\pm$0.004} & 0.951\,{\tiny$\pm$0.002} & 0.982\,{\tiny$\pm$0.002} & 1.000\,{\tiny$\pm$0.000} & 1.000\,{\tiny$\pm$0.000} & 0.427\,{\tiny$\pm$0.083} \\
    Lisa & 0.532\,{\tiny$\pm$0.032} & 0.096\,{\tiny$\pm$0.003} & 0.195\,{\tiny$\pm$0.014} & 0.438\,{\tiny$\pm$0.039} & 0.739\,{\tiny$\pm$0.039} & 0.913\,{\tiny$\pm$0.016} & 0.958\,{\tiny$\pm$0.010} & 0.997\,{\tiny$\pm$0.002} & 1.000\,{\tiny$\pm$0.000} & 0.290\,{\tiny$\pm$0.041} \\
    SaLoRA & 0.656\,{\tiny$\pm$0.016} & 0.138\,{\tiny$\pm$0.011} & 0.291\,{\tiny$\pm$0.030} & 0.586\,{\tiny$\pm$0.026} & 0.841\,{\tiny$\pm$0.010} & 0.947\,{\tiny$\pm$0.004} & 0.982\,{\tiny$\pm$0.005} & 1.000\,{\tiny$\pm$0.000} & 1.000\,{\tiny$\pm$0.000} & \textbf{0.655}\,{\tiny$\pm$0.061} \\
    SafeLoRA & \underline{0.358}\,{\tiny$\pm$0.012} & 0.099\,{\tiny$\pm$0.006} & 0.180\,{\tiny$\pm$0.009} & 0.307\,{\tiny$\pm$0.007} & \textbf{0.502}\,{\tiny$\pm$0.014} & \textbf{0.720}\,{\tiny$\pm$0.002} & \textbf{0.826}\,{\tiny$\pm$0.002} & \textbf{0.947}\,{\tiny$\pm$0.009} & \textbf{0.992}\,{\tiny$\pm$0.007} & 0.135\,{\tiny$\pm$0.005} \\
    SafeGrad & 0.495\,{\tiny$\pm$0.026} & \textbf{0.083}\,{\tiny$\pm$0.007} & 0.166\,{\tiny$\pm$0.014} & 0.385\,{\tiny$\pm$0.039} & 0.718\,{\tiny$\pm$0.055} & 0.894\,{\tiny$\pm$0.028} & 0.954\,{\tiny$\pm$0.005} & 0.991\,{\tiny$\pm$0.006} & 1.000\,{\tiny$\pm$0.000} & 0.542\,{\tiny$\pm$0.083} \\
    AsFT & \textbf{0.348}\,{\tiny$\pm$0.006} & 0.091\,{\tiny$\pm$0.002} & \underline{0.156}\,{\tiny$\pm$0.001} & \textbf{0.272}\,{\tiny$\pm$0.004} & \underline{0.504}\,{\tiny$\pm$0.014} & \underline{0.737}\,{\tiny$\pm$0.013} & \underline{0.858}\,{\tiny$\pm$0.009} & \underline{0.951}\,{\tiny$\pm$0.015} & \underline{0.995}\,{\tiny$\pm$0.005} & 0.138\,{\tiny$\pm$0.016} \\
    Ours & 0.360\,{\tiny$\pm$0.052} & \underline{0.090}\,{\tiny$\pm$0.010} & \textbf{0.154}\,{\tiny$\pm$0.005} & \underline{0.284}\,{\tiny$\pm$0.018} & 0.540\,{\tiny$\pm$0.075} & 0.780\,{\tiny$\pm$0.118} & 0.876\,{\tiny$\pm$0.079} & 0.977\,{\tiny$\pm$0.027} & 0.997\,{\tiny$\pm$0.005} & 0.267\,{\tiny$\pm$0.110} \\
    \bottomrule
  \end{tabular}}
\end{table}

%% file: tables/neuripsPapers/profiling.tex
\begin{table}[t]
  \centering
  \small
  \caption{Per-step training cost on Qwen-3.5 4B / SST2 (one A100, batch=10). \textit{step} is the median wall-clock time over 25 measured optimizer steps after a 5-step warmup. \textit{$\times$SFT} normalizes the step time by SFT. }
  \label{tab:profile-qwen3.5_4b_sst2}
  \begin{tabular}{lcccccc}
    \toprule
    Defense & step (ms) & p90 (ms) & GPU (MB) & samples/s &   $\times$SFT \\
    \midrule
    SFT & 136.8 & 160.5 & 14512 & 28.16 &  1.00 \\
    Lisa & 138.6 & 153.3 & 14797 & 28.30 &   1.00 \\
    SafeGrad & 334.7 & 371.5 & 26137 & 11.83 &  2.38 \\
    AsFT & 525.6 & 551.4 & 16083 & 7.55 &   3.73 \\
    Ours & 282.2 & 306.1 & 16739 & 13.98 &  2.01 \\
    \bottomrule
  \end{tabular}
\end{table}

%% file: tables/neuripsPapers/primal_dual.tex
\begin{table}[t]
  \centering
  \small
  \caption{Primal-Dual ablation: we compare solving the chance-constrained finetuning problem using our method and a primal dual method on the Qwen 3.5 4B model. \textbf{Bold} = best.}
  \label{tab:primal-dual}
  \begin{tabular}{clccc}
    \toprule
    Task & Method & HS $\downarrow$ & HS$_{q90}$ $\downarrow$ & FA $\uparrow$ \\
    \midrule
    \multirow{2}{*}{AG News}
     & Our method & \textbf{0.324} & \textbf{0.688} & 0.930 \\
     & Primal-Dual & 0.346 & 0.755 & \textbf{0.935} \\
    \midrule
    \multirow{2}{*}{GSM8K}
     & Our method & \textbf{0.303} & \textbf{0.645} & 0.140 \\
     & Primal-Dual & 0.380 & 0.802 & \textbf{0.155} \\
    \midrule
    \multirow{2}{*}{SST-2}
     & Our method & 0.317 & \textbf{0.684} & 0.940 \\
     & Primal-Dual & \textbf{0.311} & 0.703 & 0.940 \\
    \bottomrule
  \end{tabular}
\end{table}

%% file: tables/neuripsPapers/antibody-table.tex
\renewcommand{\arraystretch}{1.4}
\begin{table*}[t]
  \centering\small
  \caption{Antibody vs.\ our method under the harmful fine-tuning attack ($p{=}0.5$). Rows above the rule are the pre-attack floors. HS and HS$_{qK}$ (per-prompt $K$-th quantile of max harm-prob; $K{=}90$ Qwen, $K{=}50$ Llama) are lower-is-safer; FA is task accuracy. Mean$\pm$std over seeds $\{42,1337,2024\}$. \textbf{Bold} = best defense per column.}
  \label{tab:antibody_appendix}
  \resizebox{\textwidth}{!}{%
  \begin{tabular}{clccccccccc}
    \toprule
     & & \multicolumn{3}{c}{SST-2} & \multicolumn{3}{c}{AG News} & \multicolumn{3}{c}{GSM8K} \\
    \cmidrule(lr){3-5}\cmidrule(lr){6-8}\cmidrule(lr){9-11}
     & Method & HS$\downarrow$ & HS$_{qK}\downarrow$ & FA$\uparrow$ & HS$\downarrow$ & HS$_{qK}\downarrow$ & FA$\uparrow$ & HS$\downarrow$ & HS$_{qK}\downarrow$ & FA$\uparrow$ \\
    \midrule
    \multirow{6}{*}{\rotatebox[origin=c]{90}{Qwen-3.5-4B (q90)}}
     & SFT-Aligned & $0.350$ & $0.738$ & $0.940$ & $0.350$ & $0.738$ & $0.890$ & $0.350$ & $0.738$ & $0.170$ \\
     & Antibody-Aligned & $0.204$ & $0.571$ & $0.835$ & $0.204$ & $0.571$ & $0.805$ & $0.204$ & $0.571$ & $0.145$ \\
    \cmidrule(l){2-11}
     & SFT & $0.546{\scriptstyle\,\pm0.042}$ & $0.890{\scriptstyle\,\pm0.015}$ & $0.973{\scriptstyle\,\pm0.016}$ & $0.545{\scriptstyle\,\pm0.043}$ & $0.901{\scriptstyle\,\pm0.025}$ & $0.890{\scriptstyle\,\pm0.015}$ & $0.637{\scriptstyle\,\pm0.047}$ & $0.945{\scriptstyle\,\pm0.017}$ & $0.633{\scriptstyle\,\pm0.064}$ \\
     & Antibody & $\mathbf{0.273}{\scriptstyle\,\pm0.016}$ & $\mathbf{0.690}{\scriptstyle\,\pm0.016}$ & $0.867{\scriptstyle\,\pm0.062}$ & $\mathbf{0.244}{\scriptstyle\,\pm0.010}$ & $\mathbf{0.648}{\scriptstyle\,\pm0.024}$ & $0.888{\scriptstyle\,\pm0.020}$ & $0.519{\scriptstyle\,\pm0.067}$ & $0.907{\scriptstyle\,\pm0.031}$ & $\mathbf{0.383}{\scriptstyle\,\pm0.070}$ \\
     & Ours on SFT-Aligned & $0.340{\scriptstyle\,\pm0.032}$ & $0.726{\scriptstyle\,\pm0.041}$ & $\mathbf{0.957}{\scriptstyle\,\pm0.018}$ & $0.340{\scriptstyle\,\pm0.017}$ & $0.732{\scriptstyle\,\pm0.039}$ & $\mathbf{0.927}{\scriptstyle\,\pm0.006}$ & $0.360{\scriptstyle\,\pm0.052}$ & $\mathbf{0.780}{\scriptstyle\,\pm0.118}$ & $0.267{\scriptstyle\,\pm0.110}$ \\
     & Ours on Antibody-Aligned & $0.378{\scriptstyle\,\pm0.009}$ & $0.793{\scriptstyle\,\pm0.007}$ & $0.905{\scriptstyle\,\pm0.020}$ & $0.357{\scriptstyle\,\pm0.007}$ & $0.786{\scriptstyle\,\pm0.003}$ & $0.842{\scriptstyle\,\pm0.070}$ & $\mathbf{0.347}{\scriptstyle\,\pm0.013}$ & $0.796{\scriptstyle\,\pm0.033}$ & $0.198{\scriptstyle\,\pm0.039}$ \\
    \midrule
    \multirow{6}{*}{\rotatebox[origin=c]{90}{Qwen-3.5-9B (q90)}}
     & SFT-Aligned & $0.546$ & $0.781$ & $0.980$ & $0.546$ & $0.781$ & $0.855$ & $0.546$ & $0.781$ & $0.250$ \\
     & Antibody-Aligned & $0.224$ & $0.582$ & $0.785$ & $0.224$ & $0.582$ & $0.930$ & $0.224$ & $0.582$ & $0.145$ \\
    \cmidrule(l){2-11}
     & SFT & $0.670{\scriptstyle\,\pm0.023}$ & $0.898{\scriptstyle\,\pm0.031}$ & $0.980{\scriptstyle\,\pm0.005}$ & $0.672{\scriptstyle\,\pm0.025}$ & $0.885{\scriptstyle\,\pm0.016}$ & $0.903{\scriptstyle\,\pm0.030}$ & $0.681{\scriptstyle\,\pm0.020}$ & $0.947{\scriptstyle\,\pm0.005}$ & $0.300{\scriptstyle\,\pm0.080}$ \\
     & Antibody & $\mathbf{0.330}{\scriptstyle\,\pm0.030}$ & $\mathbf{0.699}{\scriptstyle\,\pm0.016}$ & $0.857{\scriptstyle\,\pm0.003}$ & $\mathbf{0.217}{\scriptstyle\,\pm0.020}$ & $\mathbf{0.596}{\scriptstyle\,\pm0.015}$ & $0.950{\scriptstyle\,\pm0.018}$ & $0.492{\scriptstyle\,\pm0.127}$ & $0.875{\scriptstyle\,\pm0.061}$ & $\mathbf{0.510}{\scriptstyle\,\pm0.038}$ \\
     & Ours on SFT-Aligned & $0.539{\scriptstyle\,\pm0.024}$ & $0.836{\scriptstyle\,\pm0.048}$ & $\mathbf{0.957}{\scriptstyle\,\pm0.013}$ & $0.536{\scriptstyle\,\pm0.022}$ & $0.822{\scriptstyle\,\pm0.030}$ & $0.902{\scriptstyle\,\pm0.050}$ & $0.548{\scriptstyle\,\pm0.027}$ & $0.857{\scriptstyle\,\pm0.059}$ & $0.325{\scriptstyle\,\pm0.087}$ \\
     & Ours on Antibody-Aligned & $0.444{\scriptstyle\,\pm0.031}$ & $0.758{\scriptstyle\,\pm0.022}$ & $0.863{\scriptstyle\,\pm0.014}$ & $0.364{\scriptstyle\,\pm0.004}$ & $0.740{\scriptstyle\,\pm0.016}$ & $\mathbf{0.958}{\scriptstyle\,\pm0.008}$ & $\mathbf{0.341}{\scriptstyle\,\pm0.041}$ & $\mathbf{0.777}{\scriptstyle\,\pm0.021}$ & $0.347{\scriptstyle\,\pm0.010}$ \\
    \midrule
    \multirow{6}{*}{\rotatebox[origin=c]{90}{Llama-3.1-8B (q50)}}
     & SFT-Aligned & $0.786$ & $0.781$ & $0.830$ & $0.786$ & $0.781$ & $0.850$ & $0.786$ & $0.781$ & $0.175$ \\
     & Antibody-Aligned & $0.426$ & $0.351$ & $0.775$ & $0.426$ & $0.351$ & $0.985$ & $0.426$ & $0.351$ & $0.110$ \\
    \cmidrule(l){2-11}
     & SFT & $0.910{\scriptstyle\,\pm0.003}$ & $0.913{\scriptstyle\,\pm0.005}$ & $0.995{\scriptstyle\,\pm0.005}$ & $0.915{\scriptstyle\,\pm0.005}$ & $0.914{\scriptstyle\,\pm0.004}$ & $0.980{\scriptstyle\,\pm0.013}$ & $0.894{\scriptstyle\,\pm0.017}$ & $0.888{\scriptstyle\,\pm0.013}$ & $0.355{\scriptstyle\,\pm0.071}$ \\
     & Antibody & $\mathbf{0.470}{\scriptstyle\,\pm0.012}$ & $\mathbf{0.378}{\scriptstyle\,\pm0.011}$ & $0.997{\scriptstyle\,\pm0.003}$ & $\mathbf{0.508}{\scriptstyle\,\pm0.008}$ & $\mathbf{0.410}{\scriptstyle\,\pm0.010}$ & $0.965{\scriptstyle\,\pm0.013}$ & $0.721{\scriptstyle\,\pm0.021}$ & $0.689{\scriptstyle\,\pm0.035}$ & $\mathbf{0.333}{\scriptstyle\,\pm0.021}$ \\
     & Ours on SFT-Aligned & $0.776{\scriptstyle\,\pm0.015}$ & $0.785{\scriptstyle\,\pm0.009}$ & $0.985{\scriptstyle\,\pm0.015}$ & $0.785{\scriptstyle\,\pm0.026}$ & $0.794{\scriptstyle\,\pm0.015}$ & $0.928{\scriptstyle\,\pm0.032}$ & $0.794{\scriptstyle\,\pm0.032}$ & $0.796{\scriptstyle\,\pm0.022}$ & $0.268{\scriptstyle\,\pm0.050}$ \\
     & Ours on Antibody-Aligned & $0.635{\scriptstyle\,\pm0.008}$ & $0.546{\scriptstyle\,\pm0.017}$ & $\mathbf{0.998}{\scriptstyle\,\pm0.003}$ & $0.641{\scriptstyle\,\pm0.012}$ & $0.565{\scriptstyle\,\pm0.002}$ & $\mathbf{0.985}{\scriptstyle\,\pm0.005}$ & $\mathbf{0.543}{\scriptstyle\,\pm0.020}$ & $\mathbf{0.443}{\scriptstyle\,\pm0.017}$ & $0.300{\scriptstyle\,\pm0.035}$ \\
    \bottomrule
  \end{tabular}}
\end{table*}